\documentclass[letterpaper, 10 pt, journal]{IEEEtran}

\let\labelindent\relax
\usepackage{graphicx} % for pdf, bitmapped graphics files
\usepackage{amsthm}
\usepackage{amsmath} % assumes amsmath package installed
\usepackage{txfonts}
\usepackage{color}
\usepackage{algorithm}
\usepackage{algpseudocode}
\usepackage{comment}
\usepackage{color}
\usepackage{cite}
\usepackage{setspace}
\usepackage{enumitem}
\usepackage{comment}
\usepackage{relsize}
\usepackage[flushleft]{threeparttable}
\usepackage{tabularx}
\usepackage{url}

\newtheorem{Assumption}{Assumption}
\newtheorem{Theorem}{Theorem}

\newtheorem{Proposition}{Proposition}
\newtheorem{Definition}{Definition}

\title{\LARGE \bf
Real-Time Area Coverage and Target Localization using Receding-Horizon Ergodic Exploration
}

\author{Anastasia Mavrommati, Emmanouil Tzorakoleftherakis, Ian Abraham and Todd D. Murphey% <-this % stops a space
\thanks{Anastasia Mavrommati, Emmanouil Tzorakoleftherakis, Ian Abraham and Todd D. Murphey are with the Department of Mechanical Engineering, Northwestern University, 2145 Sheridan Road Evanston, IL 60208, USA 
        {\tt\small Email: stacymav@u.northwestern.edu; man7therakis@u.northwestern.edu; ianabraham2015@u.northwestern.edu;  t-murphey@northwestern.edu}}%        
\thanks{ This material is based upon work supported by Army Research Office grant W911NF-14-1-0461 and by the National Science Foundation under awards CMMI-1200321 and IIS-1426961. Any opinions, findings, and
conclusions or recommendations expressed in this material are those of
the author(s) and do not necessarily reflect the views of the National
Science Foundation.}
}

\begin{document}

%\bstctlcite{IEEEexample:BSTcontrol}

\maketitle
\thispagestyle{empty}
\pagestyle{empty}
\newcommand{\argmin}{\operatornamewithlimits{argmin}}

%%%%%%%%%%%%%%%%%%%%%%%%%%%%%%%%%%%%%%%%%%%%%%%%%%%%%%%%%%%%%%%%%%%%%%%%%%%%%%%%
\begin{abstract}
	
Although a number of solutions exist for the problems of coverage, search and target localization---commonly addressed separately---whether there exists a unified strategy that addresses these objectives in a coherent manner without being application-specific remains a largely open research question. In this paper, we develop a receding-horizon ergodic control approach, based on hybrid systems theory, that has the potential to fill this gap. The nonlinear model predictive control algorithm plans real-time motions that optimally improve ergodicity with respect to a distribution defined by the expected information density across the sensing domain.  We establish a theoretical framework for global stability guarantees with respect to a distribution.  Moreover, the approach is distributable across multiple agents, so that each agent can independently compute its own control while sharing statistics of its coverage across a communication network.  We demonstrate the method in both simulation and in experiment in the context of target localization, illustrating that the algorithm is independent of the number of targets being tracked and can be run in real-time on computationally limited hardware platforms.
\end{abstract}

%%%%%%%%%%%%%%%%%%%%%%%%%%%%%%%%%%%%%%%%%%%%%%%%%%%%%%%%%%%%%%%%%%%%%%%%%%%%%%%%
\section{INTRODUCTION}
  
This paper considers the problem of real-time motion planning for area search/coverage and target localization. Although the above operations are often considered separately, they essentially all share a common objective:  tracking a specified distribution of information across the terrain.  
Our approach differs from common solutions of space grid decomposition in area coverage \cite{6889862, stergiopoulos2013spatially, soltero2013decentralized, meuth2009adaptive, schwager2009decentralized} and/or information maximization in target localization \cite{passerieux1998optimal,skoglar2009information,bishop2008optimal,liao2004application,bourgault2003optimal,ponda2009trajectory,oshman1999optimization, grocholsky2003information} by employing the metric of \textit{ergodicity} to plan trajectories with spatial statistics that match the terrain spatial distribution in a  continuous manner. By following this approach, we can establish a unified framework that achieves simultaneous search and localization of multiple targets (e.g., localizing detected targets while searching for new targets when the number of total targets is unknown) without added complexity.    Previous work \cite{miller2015ergodic, o2015optimal,miller2015optimal} has suggested using ergodic control for the purpose of motion planning for search and localization (albeit separately). However, due to its roots to optimal control theory, ergodic control has been associated with high computational cost that makes it impractical for real-time operations with varying information distribution. The contribution of this work is a model predictive ergodic control algorithm that exhibits low execution times even for high-dimensional systems with complex nonlinear dynamics while  providing stability guarantees over both dynamic states and information evolution. To the best of our knowledge, this paper includes the first application of an online ergodic control approach in real-time experimentation---here, using the \texttt{sphero SPRK} robot \cite{sphero}.  

Ergodic theory relates the time-averaged behavior of a system to the set of all possible states of the system and is primarily used in the study of fluid mixing and communication. We use ergodicity to compare the statistics of a search trajectory to a terrain spatial distribution---the distribution may represent probability of detection for area search, regions of interest for area coverage and/or expected information density for target localization. To formulate the ergodic control algorithm, we employ hybrid systems theory to analytically compute the next control action that optimally improves ergodicity, in a receding-horizon manner. The algorithm---successfully applied to autonomous visual rendering in \cite{wafr}---is shown to be particularly efficient in terms of execution time and capable of handling high-dimensional systems. The overall contribution of this paper combines the following features in a coherent manner.

\textit{Real-time execution}: 
  
 In Sections \ref{ex1} and \ref{ex2}, we demonstrate in simulation real-time reactive control for a quadrotor model in $\text{SE} (3)$.  In Section~\ref{experiment}, we show in experimentation how an ergodically controlled \texttt{sphero SPRK} robot can localize projected targets in real time using bearing-only measurements. 

\textit{Adaptive performance}:  
The proposed algorithm reactively adapts to a changing distribution of information across the workspace. We show how a UAV can adapt to new target estimates as well as to increasing numbers of targets. 

\textit{Nonlinear agent dynamics}:   As opposed to geometric point-to-point approaches \cite{meuth2009adaptive, broderick2014optimal, garzon2015multirobot}, the proposed algorithm  controls agents with complex nonlinear dynamics, such as robotic fish \cite{tan2011autonomous, castano}, 12-dimensional unmanned aerial vehicles (UAVs) \cite{quigley2005target, 6178054}, etc., taking advantage of their dynamical features to achieve efficient area exploration. 

\textit{Stability of information states}: We establish  requirements for ergodic stability of the closed-loop system resulting from the receding-horizon strategy in Section~\ref{stab}.

\textit{Multi-objective control capacity}: The proposed algorithm can work in a shared control scenario by wrapping around controllers that implement other objectives. This dual control problem solution is achieved by encoding the non-information related (secondary mission) control signal as the nominal control input $u_i^{nom }$ in the algorithm (see Algorithm~\ref{rhee}). In the simulation examples, we show how this works by wrapping the ergodic control algorithm around a PD controller for height regulation in order to control a quadrotor to explore a terrain. 

\textit{Multi-agent distributability}: We show how ergodic control is distributable to $N>1$ agents with no added computational complexity. Each agent computes control actions \textit{locally} using their own processing unit but still shares information \textit{globally} with the other agents after each algorithm iteration. 

\textit{Generalizability and robustness in multiple-targets localization}: 
The complexity of sensor motion planning strategies usually scales with respect to the total number of targets, because a separate motion needs to be planned for each target \cite{tang2005motion, summers2009coordinated}. As opposed to this, the proposed ergodic control approach  controls the robotic agents to track a non-parameterized information distribution across the terrain instead of individual targets independently, thus being completely decoupled from the estimation process and independent of the number of targets.     
Through a series of simulation examples and Monte Carlo experimental trials, we show that ergodically controlled agents with limited sensor ranges can reliably detect and localize static and moving targets in challenging situations where a) the total number of targets is unknown;  b) a model of prior target behavior is not available; c) agents acquire bearing-only measurements; d) a standard Extended Kalman Filter (EKF) is used for bearing-only estimation.  

\textit{Joint area coverage and target localization}: Planning routes that simultaneously optimize the conflicting objectives of search and tracking is particularly challenging  \cite{jilkov2009fusion, pitre2012uav, pimenta2009simultaneous}. Here, we propose an ergodic control approach where the combined dual objective is encoded in the spatial information distribution that the statistics of the robotic agents trajectories must match.

\section{Related Work}

\subsection{Area Search and Coverage}
\label{back_cov}
An area search function is required by many operations, including search-and-rescue \cite{wong2011multiple,liu2013robotic}, hazard detection \cite{garzon2015multirobot}, agricultural spraying \cite{barrientos2011aerial}, solar radiation tracking \cite{plonski2013energy} and monitoring of environmental phenomena, such as water quality in lakes \cite{singh2009efficient}. In addition, complete area coverage navigation that requires the agent to pass through every region of the workspace \cite{galceran2013survey} is an essential issue for cleaning robots \cite{luo2008bioinspired, viet2013ba},  autonomous underwater covering vehicles \cite{yilmaz2008path, paull2013sensor}, demining robots \cite{acar2003path}, automated harvesters \cite{bac2014harvesting}, etc.  Although slightly different in concept, both applications---search and coverage---involve motion planning for tracking a distribution of information across the terrain. For purposes of area search, this terrain spatial distribution indicates probability of detection and usually exhibits high variability in both space and time as it dynamically responds to new information about the target's state. In area coverage applications, on the other hand, the terrain distribution shows regions of interest and is normally near-uniform with possible occlusions.

A number of contributions in the area of robotic search and coverage decompose the exploration space to reduce problem complexity. Grid division methods of various geometries, such as Voronoi divisions \cite{6889862, stergiopoulos2013spatially, soltero2013decentralized, meuth2009adaptive,schwager2009decentralized}, are  commonly employed to accomplish this. While these methods work well, their scalability to more complex and larger terrains where the number of discrete divisions increases, is a concern. In addition, existing methods plan paths that do not satisfy the robotic agents' dynamics and thus are not feasible system trajectories. This raises the need for an additional step where the path is treated as a series of waypoints and the agent is separately controlled to visit them all \cite{meuth2009adaptive, broderick2014optimal, garzon2015multirobot}. This double-step process---first, path planning and then, robot control---might result in hard-to-accomplish maneuvers for the robotic system. Finally, decomposition methods often do not respond well to dynamically changing environments---for example when probability of detection across the workspace changes in response to incoming data---because grid updates can be computationally intensive. Therefore, most existing solutions only perform non-adaptive path planning for search and coverage offline---i.e., when the distribution of information is known and constant\footnote{To overcome this issue when monitoring environments with changing distributions, an alternative solution is to control only the speed of the robotic agents over a predefined path \cite{smith2012persistent}.}. As opposed to this, the algorithm described in this paper does not require decomposition of the workspace, by representing probability of detection as a continuous function over the terrain. Furthermore, it performs online motion planning by reactively responding to changes in the terrain spatial distribution in real time, while taking into account the agent's dynamics.

\textit{Multi-agent Coordinated Coverage:}

The objective of multi-agent coverage control  is to  control the
agents’ motion in order to collectively track a spatial probability-of-detection density function \cite{cao2013overview} across the terrain.  Shared information is a necessary condition for coordination \cite{ren_multi}.  Several promising coverage control algorithms for mobile sensor networks have been proposed. In most cases, the objective is to control the agents to move to a static  position that optimizes detection probability $\Phi$, i.e., to compute and track the final states $x_{\zeta}(t_f)$   that   maximize the sum of  \mbox{$\int \mathcal{S}(\|x-x_{\zeta}(t_f)\|) \Phi(x) dx $} over all agents $\zeta$ where $\mathcal{S}$ indicates sensor performance. Voronoi-based gradient descent \cite{cortes2005spatially, kantaros2015distributed}  is a popular approach but it can converge to local maxima. Other approaches employ cellular decomposition \cite{rutishauser2009collaborative}, simulated annealing  \cite{kwok2011distributed}, game theory \cite{zhu2013distributed, 7383274} and ergodic dynamics \cite{shell2006ergodic} to achieve distributed coverage.  The main drawback is that existing algorithms do not consider time-dependent density functions $\Phi$, so they are not suitable for realistic applications where probability of detection varies. 
Ergodic multi-agent coverage described in this paper differs from common coverage solutions in that it aims to  control the agents so that the spatial statistics of their full trajectories---instead of solely their final states---optimize detection probability, i.e., the time-averaged integral $C(x) = \frac{1}{t_f-t_0}\int_{t_0}^{t_f} \delta[x-x_{\zeta}(t)] dt$, where $\delta$ is the Dirac delta,  matches the spatial distribution $\Phi(x)$  as $t_f \rightarrow \infty$. This means that if we capture a single snapshot of the agents' ergodic motion,  there is no guarantee that their current configuration will be maximizing the probability density. However, as time progresses the network of agents is bound to explore the terrain as extensively as possible by being ergodic with respect to the terrain distribution. An advantage of this approach compared to other coordinated coverage solutions, is that it can be performed online in order to cover terrains with time-dependent (or sensed in real time) density functions.

\subsection{Motion Planning for Target Localization }

\label{back_loc}

Target localization\footnote{While often ``localization'' refers to static targets and ``tracking'' is used for moving targets, here for consistency we will use the term ``localization'' to describe estimation of both static and moving targets. } refers to the process of acquiring and using sensor measurements to estimate the state of a single or multiple targets. One of the main challenges involved with target localization is developing a sensor trajectory such that high information measurements are acquired. To achieve this, some methods perform information maximization (IM) \cite{passerieux1998optimal,skoglar2009information,bishop2008optimal,liao2004application,bourgault2003optimal,ponda2009trajectory,oshman1999optimization, julian2012distributed} usually by compressing an information metric (such as Fisher Information Matrix (FIM) \cite{6178054,wilson2017dynamic} and  Bayesian utility functions \cite{bourgault2003optimal}) to a scalar cost function (e.g., using the determinant \cite{oshman1999optimization}, trace \cite{ponda2009trajectory}, or eigenvalue \cite{leung2005trajectory, 7570171} of the information matrix)  and then generating trajectories that optimize this cost. The most general approaches to solving these problems involve numerical techniques \cite{helferty1993optimal}, classical optimal control theory \cite{passerieux1998optimal}, and  dynamic programming \cite{singh2009efficient,cao2013multi}, which tend to be either  computationally intensive or application specific (e.g., consider only static and/or constant velocity targets). Compared to IM techniques \cite{passerieux1998optimal,skoglar2009information,bishop2008optimal,liao2004application,bourgault2003optimal,ponda2009trajectory,oshman1999optimization, julian2012distributed}, the proposed algorithm   explores an expected information density map by optimally improving ergodicity, thus being more time-efficient and more robust to the existence of local maxima in information. In addition, it allows tracking multiple moving targets without added complexity as opposed to most IM techniques that would need to plan a motion for each target separately. To overcome this issue with IM techniques, Dames et al. \cite{dames2015detecting, dames2015autonomous} propose an estimation filter that estimates the targets' density---instead of individual labeled targets---over time,  thus rendering IM complexity independent of the number of targets. However, the proposed trajectory generation methodology relies on exhaustive search requiring  discretization of the controls space.

Because IM techniques tend to exhibit prohibitive  execution times for moving targets, alternative methods of diverse nature have been proposed for use in real-world applications. A non-continuous grid decomposition strategy for  planning parameterized paths for UAVs is proposed in \cite{yu2015cooperative} with the objective to localize a single target by maximizing the probability of detection when the target motion is modeled as a Markov process. Standoff tracking techniques are commonly used to control the agent to achieve a desired standoff configuration from the target usually by orbiting around it \cite{kim2013nonlinear, anderson2011stochastic, summers2009coordinated}. A probabilistic planning approach  for localizing a group of targets using vision sensors is detailed in \cite{spletzer2003dynamic}. In \cite{zhu2013ground}, a UAV is used to track a target in constant wind considering control input constrains, but the planned path is predefined to converge  to a desired circular orbit around the target. A rule-based guidance strategy
for localizing moving targets is introduced in \cite{zengin2007real}. In \cite{yao2015real}, the problem of real-time path planning for tracking a single target while avoiding obstacles is addressed through a combination of methodologies. In general, as opposed to ergodic exploration, the above approaches focus on and are only applicable in special real-world situations and do not generalize directly to general multiple-target tracking situations.  For a complete and extensive comparison of ergodic localization to other motion planning approaches, the reader is referred to \cite{miller2015ergodic}.

In this paper, we use as an example the problem of bearing-only localization. Many real-world systems use angle-only sensors for target localization, such as submarines with passive sonar, mobile robots with  directional radio antenna \cite{tokekar2011active}, and unmanned aerial vehicles (UAVs) using images from an optical sensor  \cite{ross2014stochastic}. Bearing-only systems require some amount of maneuver to measure range with minimum uncertainty \cite{passerieux1998optimal}. The majority of existing solutions for UAV bearings-only target motion planning, produce circular trajectories above the target's  estimated position \cite{barber2006vision, quigley2005target}. However, this solution is only viable if there is low uncertainty over the target's position. In addition, if the target is moving, the operator may not know what the best vehicle path is. In this paper, this drawback is overcome by exploring an expected information density that is updated in real time while the targets are moving based on the current targets' estimate.  

\textit{Cooperative Target Localization:} 
The  greatest body of work in the area of cooperative target localization is comprised by standoff tracking techniques. Vector fields \cite{summers2009coordinated}, nonlinear feedback \cite{ma2013cooperative} and nonlinear model predictive control \cite{kim2013nonlinear} are some of the control methodologies that have been used for achieving the desired standoff configuration for a target.  The motion of the robotic agents is coordinated in a geometrical manner: two robotic agents orbit the target at a nominal standoff distance and maintain a specified angular separation; when more agents are considered, they are controlled to achieve a uniform angular separation on a circle around the target.  A dynamic programming technique that minimizes geolocation error covariance is proposed in  \cite{quintero2016stochastic}. However, the solution is not generalizable to multiple robotic agents and targets.  The robots are controlled to seek informative observations
by moving along the gradient of mutual information in \cite{julian2012distributed}. The main concern in using these strategies is scalability to multiple targets, as the robots are separately controlled to fly around each single target \cite{gu2006optimal}. To overcome this issue, we propose an algorithm that tracks a non-parameterized information density across the terrain instead of individual targets independently, thus being completely decoupled from the estimation process and the number of targets.

\begin{table*}[!t]
	\renewcommand{\arraystretch}{1.3}
	\caption{List of Variables}
	\label{table_example}
	\centering
	\resizebox{6.9in}{!}{
	\begin{tabular}{|||c | c ||| c | c ||| c | c |||}
		\hline
		\bfseries Symbol & \bfseries Description & \bfseries Symbol & \bfseries  Description  & \bfseries Symbol & \bfseries Description\\
		\hline\hline
		$x$ & system dynamic states &	$J_{\mathcal{E}}$ & receding-horizon ergodic cost & $u^*_{s}$ & schedule of candidate infinitesimal actions \\
		$u$ & system inputs & $ t_0^{erg}$ & initial time of ergodic exploration & $M_{erg}$ & ergodic memory\\
		$n$ & number of dynamic states & $t_s$ & algorithm sampling time & 	$N$ & number of agents on the field\\
		$m$ & number of system inputs & $ Q$ & weight on ergodic cost & 	$\zeta$ & agents performing exploration \\
		$\nu$ & number of ergodically explored states ($\nu \leq n$) & 	$ R $ & weight on control effort & $M$ & number of single-target coordinates\\
		$ L_i$ & $i$-th dimension of search domain & $\rho_{\mathcal{E}}$ & ergodic costate/adjoint & 	$\boldsymbol{\alpha}$ & single-target parameters to be estimated\\
		$\Phi(\cdot)$ & spatial distribution on the search domain & 	$ \alpha_d $ & desired rate of change & 	$z$ & sensor measurement  \\ 
		$ \phi_k$ & Fourier coefficients of $\Phi(\cdot)$ & 	$ u ^{nom} $ & nominal control & 	$\mu$ & number of sensor measured quantities \\
		$ c_k^i$ &  Parameterized trajectory spatial statistics at time step $t_i$ & 	$ u_i^*(t), x_i^*(t)$ & open-loop control and state trajectories at time step $t_i$ & 	$\Upsilon(\cdot)$ & deterministic sensor measurement model \\	
		$K$ & highest order of Fourier coefficients & 	$ u^{\text{\emph{def}}}(t) , x^{\text{\emph{def}}}(t)$ & default control and resulting state trajectory & 	$\mathcal{I}$ & Fisher Information matrix \\
		$k$ & set of $\nu$ coefficient indices $\{k_1,k_2,...,k_\nu\}$ & 	$\bar{x}_{cl}$ & closed-loop state trajectory & $\Sigma$ & covariance of measurement model\\
		$F_k$ & $k$-th Fourier basis function & 	$u_A, \lambda_A, \tau_A$ & magnitude, duration and application time of action $A$ & 	$\mathbf{\Phi}(\cdot)$ & expected information matrix \\
		$ T $ & open-loop time horizon & 	$\mathcal{C}_\mathcal{E}$ & cost contractive constraint & $\mathcal{F}(\cdot)$  & target dynamics\\
		\hline
	\end{tabular}}
\end{table*}

\subsection{Ergodic Control Algorithms}

There are a few other algorithms that perform ergodic control in the sense that they optimize the ergodicity metric in \eqref{metric}. 
Mathew and Mezi{\'c} in \cite{mathew2011metrics} derive analytic ergodic control formulas for simple linear dynamics (single and double integrator) by minimizing the Hamiltonian \cite{kirk2012optimal} in the limit as the receding time horizon goes to zero. Although closed-form, their solution is not generalizable to arbitrary nonlinear system dynamics and it also augments the state vector to include the coefficients difference so that the final system dimensionality is nominally infinite.  Miller et al. \cite{miller2013trajectory} propose an open-loop  trajectory optimization technique using a finite-time horizon. This algorithm is ideal for generating optimal  ergodic solutions with a prescribed time window. However, it exhibits relatively high computational cost that does not favor real-time algorithm application in a receding-horizon format. This approach has been used for offline receding-horizon exploration of unknown environments \cite{o2015optimal} and localization of a single static target  \cite{miller2015ergodic,miller2015optimal} in cases where real-time control is not imperative. De La Torre et al. \cite{7525371} propose a stochastic differential dynamic programming algorithm for ergodic exploration in the presence of stochastic sensor dynamics. 
The ergodic control algorithm in this paper differs from these existing approaches in that it employs hybrid systems theory to perform ergodic exploration fast, in real time, while adaptively responding to changes in the information distribution as required for tracking moving targets. 

\section{Ergodicity metric}
\label{erg_section}

 For area coverage and target localization using ergodic theory, the objective is to control an agent so that the amount of time spent in any given area of a specified search domain is proportional to the integral of a spatial distribution over that same domain. This section describes an ergodicity metric that satisfies this objective.

Consider a search domain that is a bounded $\nu$-dimensional workspace $\mathcal{X}_\nu \subset \mathbb{R}^{\nu}$ defined as $[0, L_1] \times [0, L_2] \times ... \times [0,L_\nu]$ with $\nu\leq n$, where $n$ is the total number of the system dynamic states.
If $s \in \mathcal{X}^{\nu}$ denotes a point in the search domain, the spatial distribution over the search domain is denoted  as $\Phi(s): \mathcal{X}^{\nu} \rightarrow \mathbb{R}$, and it can represent probability of detection in search area coverage operations, such as search-and-rescue, surveillance, inspection etc. or expected information density in target localization tasks as in Section~\ref{localization}.   
The spatial statistics of a trajectory $x_\nu(t)$ are quantified by the
percentage of time spent in each region of the workspace as
\begin{equation}
\label{xstat}
C(s) = \frac{1}{T} \int\limits_{t_0}^{t_0+T} \delta[s-x_\nu(t)] dt
\end{equation}
where $\delta$ is the Dirac delta.  We use the distance from ergodicity between the spatial statistics of the time-averaged trajectory and the terrain spatial distribution as a metric. To drive the spatial
statistics of a trajectory $x_\nu(t)$ to match those of the distribution
$\Phi(s)$, we need to choose a norm on the difference between
the distributions $\Phi(s)$ and $C(s)$. As in \cite{miller2015ergodic}, we quantify the difference
between the distributions, i.e., the \textit{distance from ergodicity},
using the sum of the weighted squared distance between
the Fourier coefficients $\phi_k$ of $\Phi(s)$, and the coefficients $c_k$
of the distribution $C(s)$ representing the time-averaged trajectory. In particular, the Fourier coefficients $\phi_k$ and $c_k$ are calculated respectively as
\begin{equation}
\label{phik}
\phi_k =  \int\limits_{\mathcal{X}^{\nu}} \Phi(x_\nu) F_k(x_{\nu}) dx_\nu
\end{equation} 
and
\begin{equation}
\label{info_states}
c_k = \frac{1}{T} \int\limits_{t_0}^{t_0+T} F_k(x_{\nu}(t)) dt
\end{equation} 
where $F_k$ is a Fourier basis function, as derived in \cite{mathew2011metrics}, and $T$ is the open-loop time horizon. 
Here, we use the following choice of basis function: 
\begin{equation}
F_k(s) = \frac{1}{h_k} \prod\limits_{i=1}^{\nu} \cos\bigg(\frac{k_i \pi}{L_i} s_i\bigg), 
\end{equation}
where $k \in \mathcal{K}$ is a set of $\nu$ coefficient indices $\{k_1,k_2,...,k_\nu\}$ with $k_i \in \mathbb{N}$ so that $\mathcal{K} = \{ k \in \mathbb{N}^{\nu}: 0 \leq k_i \leq K \}$,
$K \in \mathbb{N}$ is the highest order of coefficients calculated along each of the $\nu$ dimensions, and $h_k$ is a normalizing factor \cite{mathew2011metrics}. It should be noted,
however, that any set of basis functions that is differentiable in the state and can be evaluated along the trajectory can be used in the derivation of the ergodic metric. 
Using the above, the ergodic metric on $x_\nu \in \mathcal{X}^\nu$ is defined as in \cite{miller2015ergodic, miller2013trajectory,mathew2011metrics}  
\begin{equation}
\label{metric}
\mathcal{E}(x_\nu(t)) = \sum\limits_{k \in  \mathcal{K}} \Lambda_k [c_k(x_\nu(t))-\phi_k]^2
\end{equation}
with $\Lambda_k=\frac{1}{(1+||k||^2)^s}$ and $s=\frac{\nu+1}{2}$, which places larger weight on lower frequency information so that when $K \rightarrow \infty$ the series converges.

\section{Receding-horizon ergodic exploration}

\subsection{Algorithm Derivation}
\label{main_ergodic}
We shall consider nonlinear control affine systems with input constraints such that 
\begin{gather}
\label{f}
\dot x = f(t,x,u) = g(t,x) + h(t,x) \, u \;\;\;\forall t \\
\text{with\;\;} u\in\mathcal{U} \;\;\text{and}\notag\\
\begin{aligned}
\mathcal{U}:=\Big\{ u\in\mathbb{R}^m: u_{\text{\emph{min}}}\leq u \leq u_{\text{\emph{max}}},\; u_{\text{\emph{min}}} < 0 < u_{\text{\emph{max}}}\Big\} \notag \text{,}
%\phantom{\hspace{-5cm}}
\end{aligned}
\end{gather}
i.e., systems that can be nonlinear with respect to the state vector,  \mbox{$x:\mathbb{R}\to\mathcal{X}$}, but are assumed to be linear (or linearized) with respect to the control vector,  \mbox{$u:\mathbb{R}\to\mathcal{U}$}. The state will sometimes be denoted as $t \mapsto x\big(t;x(t_i), u(\cdot)\big)$ when we want to make explicit the dependence on the initial state (and time), and corresponding control signal.
Using the metric \eqref{metric}, receding-horizon ergodic control must optimally improve the following cost at each time step $t_i$:
\begin{equation}
\label{ergodic_cost}
J_{\mathcal{E}} = Q \sum\limits_{k \in  \mathcal{K}} \Lambda_k \bigg[\underbrace{\frac{1}{t_i+T-t_0^{erg}}\int\limits_{t_0^{erg}}^{t_i+T}{F_k(x(t)) dt}}_{c_k^i}-\phi_k \bigg]^2
\end{equation}
where $t_0^{erg}$ is   the user-defined initial time of ergodic exploration, $x\in \mathcal{X}^\nu$ and $Q \in \mathbb{R}$ weights the ergodic cost against control effort weighted by $R$ in \eqref{Ju}. Henceforth, for brevity we refer to the trajectory of the set of states to be ergodically explored as $x(t)$ instead of $x_\nu(t)$, although it should be clear that the ergodically explored states might or might not be all the states of the system dynamics (i.e., $\nu\leq n$).

To understand the challenges of optimizing \eqref{ergodic_cost}, we  distinguish between the dynamic states of the controlled system, $x\in \mathbb{R}^n$---e.g., the 12 states denoting position and heading, and their velocities,  in quadrotor dynamics---and the information states \mbox{$c_k(x(\cdot)) $} in \eqref{info_states}, i.e., the parameterized time-averaged statistics of the followed trajectory over a finite time duration. The main difficulty in optimizing ergodicity is that the ergodic cost functional in \eqref{ergodic_cost} is non-quadratic  and does not follow the Bolza form---consisting of a running and terminal cost \cite{liberzon2012calculus}---with respect to the dynamic states. To address this, infinite-dimensional trajectory optimization methods that are independent of the cost functional form have been employed \cite{miller2015ergodic} to optimize ergodicity. However, the computational cost of such iterative methods is prohibitive for real-time control in a receding-horizon approach. Another method involves change of coordinates so that the cost functional is rendered quadratic with respect to the information states parameterized by Fourier coefficients. This allows the use of traditional optimal control approaches e.g., LQR, DDP, SQP etc. (see for example \cite{7525371}). However, this approach entails optimization over an extended set of states (the number of parameterized information states is usually significantly  larger than the dynamic states) which inhibits real-time execution. In addition, and perhaps more importantly, defining a running cost on the information states results in unnecessarily repetitive integration of the dynamic state trajectories. To avoid this, an option would be to optimize a terminal cost only, but this proves problematic in establishing stability of Model Predictive Control (MPC) algorithms (see \cite{fontes2001general}).

 To overcome the aforementioned issues, we seek to formulate an algorithm that a) computes control actions that guarantee contraction of the ergodic cost at each time step b) naturally uses current sensor feedback to compute controls fast, in real time. For these reasons,  we choose to frame the control problem as a hybrid control problem, similarly to Sequential Action Control (SAC) in \cite{sactro, tzorakoleftherakis2016model}. By doing this, we are able to formulate an ergodic control algorithm that is rendered fast enough for real time operation---as opposed to traditional model predictive control algorithms that are usually computationally expensive \cite{mayne2014model}---for two main reasons: a) a single control action is calculated at every time step using a closed-form algebraic expression and b) this control action aims to optimally improve ergodicity (instead of optimizing it) by an amount that guarantees stability with respect to $x$ and $c_i$.

\begin{algorithm}
	\normalsize{
		%\centering
		%\framebox{\parbox{3.3in}{
		\caption{ Receding-horizon ergodic exploration (RHEE)}
		\label{rhee} 
		\hrule
		
		\vspace{2 mm}
		\textit{\textbf{Inputs:}} initial time $t_0$, initial state $x_0$, terrain spatial distribution $\Phi(x)$, ergodic initial time $t_0^{erg}$, final time $t_f$\\     
		\textit{\textbf{Output:}} closed-loop ergodic trajectory $\bar{x}_{cl}: [t_0,t_f]\rightarrow \mathcal{X}$
		
		\hrule
		\vspace{2 mm}
		
		\textbf{Define} ergodic cost weight $Q$, highest order of coefficients $K$, control weight $R$, search domain bounds $\{L_1,...,L_{\nu}\}$, sampling time $t_s$, desired rate of change $\alpha_d$, time horizon $T$. \\
		
		\textbf{Initialize} nominal control $u^{nom}$,  step $i=0$.\\  
		\setstretch{0.5}    
		\begin{itemize}[leftmargin=*]
			\item Calculate $\phi_k$ using \eqref{phik}.  \\
			\item While $t_i<t_f$\\
			\setstretch{0.2}
			\begin{itemize}
				\item Solve open-loop problem $\mathcal{P}_\mathcal{E}(t_i,\,x_{i}, T)$ to get $u_i^*$: 
				\setstretch{0.9}     
				\begin{enumerate}
					\item Simulate system \eqref{f} for $t \in [t_i,t_i+T] $ under $u_i^{def}$ to get $x(t)$. 
					\item Simulate \eqref{ergodic_rhodot} for $t \in [t_i,t_i+T] $.
					\item Compute $u^*_{s}$ using \eqref{uopt}.
					\item Determine action application time $t_A$ and value $u_A$ by minimizing \eqref{Jt}.
					\item Determine action duration $\lambda_A$ using the line search process in Section \ref{howlong} and the condition in \eqref{improve_condition} with $\mathcal{C}_{\mathcal{E}}$ in \eqref{ergodic_condition}.               
				\end{enumerate}
				
				\setstretch{0.9}
				\item Apply $u_i^*$ to \eqref{f} for $t\in [t_i,t_i+t_s]$ to get $\bar{x}_{cl} \forall t\in [t_i,t_i+t_s]$.
				\item Define $t_{i+1} =  t_i+t_s$, $x_{i+1} =  \bar{x}_{cl}(t_{i+1})$.
				\item $i \leftarrow i+1$      
			\end{itemize}
			\setstretch{0.9}
			end while
			
		\end{itemize}

	}
	% }}
\end{algorithm}

An overview of the algorithm is given in Algorithm~\ref{rhee}. Once the Fourier coefficients $\phi_k$ of the spatial distribution $\Phi(x)$ have been calculated, the  algorithm follows a receding-horizon approach; controls are obtained by repeatedly solving online an open-loop ergodic control problem $\mathcal{P}_{\mathcal{E}}$ every $t_s$ seconds (with sampling frequency 1/$t_s$), every time using the current measure of the system dynamic state $x$. The following definitions are necessary before introducing the open-loop problem.

\begin{Definition}
	An action $A$ is defined by the triplet consisting of a control's value, $u_A \in \mathcal{U}$, application duration, \mbox{$\lambda_A \in \mathbb{R^+}$} and application time, $\tau_A \in \mathbb{R}$, such that \mbox{$A:=\{u_A, \lambda_A, \tau_A\}$}. 
\end{Definition}

\begin{Definition}
	Nominal control $u^{\text{\emph{nom}}}: \mathbb{R} \rightarrow \mathcal{U}$, provides a nominal trajectory around which the algorithm provides feedback. When applying ergodic control as a standalone controller,  $u^{\text{\emph{nom}}}(\cdot)$ is either zero or constant. Alternatively, $u^{\text{\emph{nom}}}(\cdot)$ may be an optimized feedforward or state-feedback controller.
\end{Definition}

The open-loop problem $\mathcal{P_\mathcal{E}}$ that is solved in each iteration of the receding horizon strategy can now be defined as follows\footnote{From now on, subscript $i$ will denote the $i$-th time step, starting from $i=0$.}. 
\begin{align}
\label{open_loop_problem}
&\mathcal{P_\mathcal{E}}(t_i,\,x_{i}, T):\\
&\text{Find action } A \text{ such that} \notag \\
&J_\mathcal{E}\big(x(t;x_{i}, u_i^*(\cdot))\big)-J_\mathcal{E}\big(x(t;x_{i}, u_i^{\text{\emph{def}}}(\cdot))\big) < \mathcal{C}_\mathcal{E} \label{improve_condition}\\
&\text{subject to} \notag \\
&u_i^*(t) = \begin{cases} 
u_{A} & \tau_A\leq t\leq \tau_A+\lambda_A \\
u_i^{\text{\emph{def}}}(t) & \text{else}
\end{cases} \text{,} \notag\\
&\text{and }\eqref{f} \text{ with } t\in[t_i,t_i+T] \text{ and }x(t_i)=x_{i} \text{.} \notag
\phantom{\hspace{3cm}}
\end{align}
where $\mathcal{C}_\mathcal{E}$ is a quantity that guarantees stability (see Section~\ref{stab}) and $u_i^{def}$ and $x_i^{def}$ are defined below. 
\begin{figure}[t!]
	\centering
	\includegraphics[width=3.1in]{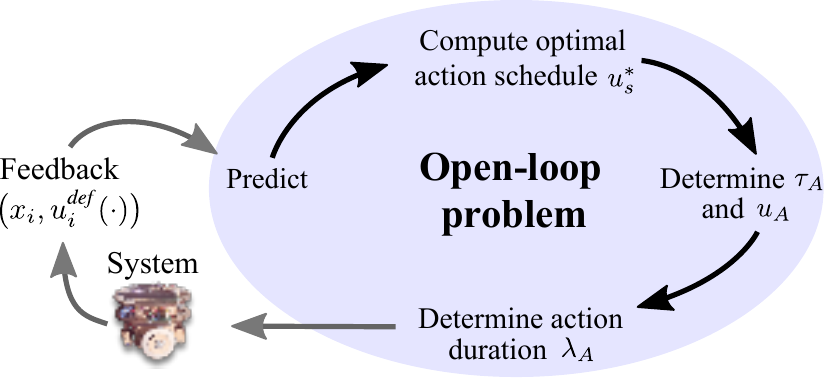}
	\caption{An overview of the ergodic control process. One major difference between the proposed ergodic control algorithm and traditional MPC approaches is that the open-loop problem can be solved without employing nonlinear programming solvers \cite{wachter2006implementation} by using hybrid systems theory. In order to solve \eqref{open_loop_problem}, the algorithm follows four steps as illustrated above. }
	\label{fig:isac_process}
\end{figure}
\begin{Definition}
	Default control \mbox{$u_i^{\text{\emph{def}}}: [t_i, t_i+T] \rightarrow \mathcal{U}$}, is defined as 
	\begin{align}
	\label{default_control}
	u_i^{\text{\emph{def}}}(t) = \begin{cases} 
	u_{i-1}^*(t) & t_i \leq t\leq t_i+T-t_s \\
	u_i^{\text{\emph{nom}}}(t) & t_i+T-t_s < t\leq t_i+T
	\end{cases} \text{,}
	\end{align}
\end{Definition}
with \mbox{$u_0^{def}(\cdot) \equiv u_0^{\text{\emph{nom}}}(\cdot)$}, $u_{i-1}^*: [t_{i-1},t_{i-1}+T] \rightarrow \mathcal{U}$ the output  of  $\mathcal{P}_{\mathcal{E}}(t_{i-1}, x_{i-1}, T)$ from the previous time step $i-1$---corresponding to $x_{i-1}^*(\cdot)$---and $t_s=t_i-t_{i-1}$ the sampling period (Fig.~\ref{fig:isac_process}). The system trajectory corresponding to application of default control will be denoted as $x\big(t;x(t_i), u^{\text{\emph{def}}}(\cdot)\big)$ or $x^{\text{\emph{def}}}(\cdot)$ for brevity.

The structure of $u_i^*(t)$ in \eqref{open_loop_problem} is a design choice that allows us to compute a single control action $A$, since $u_i^{\text{\emph{def}}}(t)$ is known at each time step $t_i$. As pointed out in the second paragraph of Section~\ref{main_ergodic}, this single-action control problem renders the algorithm fast enough for real-time execution. Moreover, the structure of default control $u_i^{\text{\emph{def}}}(t)$ in \eqref{default_control} indicates that the algorithm stores actions calculated in the previous time steps (included in $u_{i-1}^*(t)$).  By defining default control this way, and thus storing past actions, we are able to rewrite \eqref{improve_condition} as a constraint on sequential open-loop costs, as in \eqref{equiv_condition}, in order to show stability.  

The following proposition is necessary before going though the steps for solving $\mathcal{P}_{\mathcal{E}}$. 

\begin{Proposition}
	\label{new_dJ_dl}
	Consider the case where the system \eqref{f} evolves according to
	default control dynamics $f\big(t,x(t),u_i^{\text{\emph{def}}}(t)\big)$, and  action $u_A$ is applied at time $\tau$ (dynamics switch to $f\big(t,x(t),u_A)$) for an infinitesimal duration $\lambda \rightarrow 0$ before switching back to default control. In this case, the ergodic mode insertion gradient $\frac{\partial J_{\mathcal{E}}}{\partial \lambda}$ evaluated at $t=\tau$ measures the first-order sensitivity of the
	ergodic cost \eqref{ergodic_cost} to infinitesimal application of control action $u_A$ and is calculated as
	\begin{equation}
	\label{mode2}
	\frac{\partial J_{\mathcal{E}}}{\partial \lambda}\bigg|_{\tau} = \rho_{\mathcal{E}}(\tau) \; \Big[ f(\tau,x_i^{\text{\emph{def}}}(\tau),u_A)-f(\tau,x_i^{\text{\emph{def}}}(\tau),u_i^{def}(\tau)) \Big]
	\end{equation}
	with 
	\begin{gather}
	\label{ergodic_rhodot}
	\dot \rho_{\mathcal{E}} = -\ell(t,x_i^{\text{\emph{def}}})^T - D_{x}f\big(t,x_i^{\text{\emph{def}}},u_i^{\text{\emph{def}}}\big)^T \cdot \rho_{\mathcal{E}} \\
	\text{subject to\;\;} \rho_{\mathcal{E}}(t_i+T) = \mathbf{0} \notag \\
	\text{and  \;\; } \ell(t,x) = \frac{2 Q }{t_i+T-t_0^{erg}} 
	\sum\limits_{k\in  \mathcal{K}} \bigg \{\Lambda_k \big[ c_k^i-\phi_k  \big]
	\frac{\partial F_k(x(t))} {\partial x(t)}\bigg\}.\notag
	\end{gather}
	
\end{Proposition}
\begin{proof}
	The proof is provided in Appendix \ref{proof}.
\end{proof}

Note that the open loop problem $\mathcal{P_\mathcal{E}}$ in \eqref{open_loop_problem} only requires that the cost $J_{\mathcal{E}}$ is decreased by a specific amount  $\mathcal{C}_\mathcal{E}$ (at minimum), instead of maximizing the cost contraction. This allows us to quickly compute a solution, while still achieving stability, following the steps listed in the following.

\subsubsection{Predict}
\label{prediction}
 In this step, the algorithm evaluates the system \eqref{f} from the current state $x_i$ and time $t_i$, with $u_i^{\text{\emph{def}}}(t)$ for \mbox{$t \in [t_i,\;t_i+T]$}. In addition, it uses the predicted state trajectory to backward simulate $\rho_{\mathcal{E}}$ that satisfies \eqref{ergodic_rhodot}.

\subsubsection{Compute optimal action schedule $u^*_s(\cdot)$}
\label{computeoptimal}
In this step, we compute a schedule \mbox{$u^*_s:[t_i, t_i+T] \to\mathbb{R}^m$} which contains candidate infinitesimal actions. Specifically, $u^*_s(\cdot)$ contains candidate action values and their corresponding application times, but assumes $\lambda \to 0^+$ for all. The final $u_A$ and $\tau_A$ will be selected from these candidates in step three of the solution process such that \mbox{$u_{A}=u_s^*(\tau_A)$}, while a finite duration $\lambda_A$ will be selected in the final step. The optimal action schedule $u^*_s(\cdot)$ is calculated by minimizing
%k
\begin{gather}
\label{Ju}
J_{u_s} = \frac{1}{2} \int_{t_i}^{t_i+T} \bigg [\frac{\partial J_{\mathcal{E}}}{\partial \lambda}(t) - \alpha_d \bigg ]^2 + \lVert u_s(t) \rVert_{R}^2 \,dt \text{,} \\
\frac{\partial J_{\mathcal{E}}}{\partial \lambda}(t)=\rho_{\mathcal{E}}(t)^T \left[f\big(t,x_i^{\text{\emph{def}}}(t),u_s(t)\big)-f\big(t,x_i^{\text{\emph{def}}}(t),u_i^{\text{\emph{def}}}(t)\big)\right] \notag 
\end{gather}
where the quantity $\frac{\partial J_{\mathcal{E}}}{\partial \lambda}(\cdot)$ (see Proposition~\ref{new_dJ_dl}), called the mode insertion gradient \cite{egerstedt2006transition}, denotes the rate of change of the cost with respect to a switch of infinitesimal duration $\lambda$ in the dynamics of the system. In this case, $\frac{dJ_{\mathcal{E}}}{d \lambda}(\cdot)$ shows how the cost will change if we introduce a single infinitesimal switch from \mbox{$f\big(t,x_i^{\text{\emph{def}}}(t),u_i^{\text{\emph{def}}}(t)\big)$} to \mbox{$f\big(t,x_i^{\text{\emph{def}}}(t),u_s(t)\big)$} at some point in the time window \mbox{$[t_i,t_i+T]$}. The parameter $\alpha_d \in \mathbb{R}^-$ is user specified and allows the designer to influence how aggressively each action value in the schedule $u_s^*(t)$ improves the cost.

Based on the evaluation of the dynamics \eqref{f}, and \eqref{ergodic_rhodot} completed in the prediction step (Section~\ref{prediction}), minimization of \eqref{Ju} leads to the following closed-form expression for the optimal action schedule:
\begin{equation}
\small
\label{uopt}
u^*_s(t) = (\Lambda + R^T)^{-1} \, \big [\Lambda \, u_i^{\text{\emph{def}}}(t) + h\big(t,x_i^{\text{\emph{def}}}(t)\big)^T \rho_{\mathcal{E}}(t) \, \alpha_d \big ] \text{,}
\end{equation}
\normalsize
where $\Lambda \triangleq h\big(t,x_i^{\text{\emph{def}}}(t)\big)^T \rho_{\mathcal{E}}(t) \rho_{\mathcal{E}}(t)^T h\big(t,x_i^{\text{\emph{def}}}(t)\big)$. The infinitesimal action schedule can then be directly saturated to satisfy any min/max control constraints of the form \mbox{$u_{\text{\emph{min,}}k} < 0 < u_{\text{\emph{max,}}k} \; \forall k \in \{1, \dots , m\}$} such that $u^*_s \in \mathcal{U}$ without additional computational overhead (see \cite{sactro} for proof).

\subsubsection{Determine application time $\tau_A$ (and thus $u_{A}$ value)}
\label{whentoact}

Recall that the curve $u^*_s(\cdot)$ provides the values and application times of possible infinitesimal actions that the algorithm could take at different times to optimally improve system performance from that time. In this step the algorithm chooses one of these actions to apply, i.e., chooses the application time $\tau_A$ and thus an action value $u_{A}$ such that $u_{A}=u^*_s(\tau_A)$. To do that, $u^*_s(\cdot)$ is searched for a time $\tau_A$ that minimizes
\begin{gather}
\label{Jt}
J_{t}(\tau) = \frac{\partial J_{\mathcal{E}}}{\partial \lambda} \bigg \vert_{\tau} \text{,} \\
\frac{\partial J_{\mathcal{E}}}{\partial \lambda} \bigg \vert_{\tau}=\rho_{\mathcal{E}}(\tau)^T \left[f\big(\tau,x_i^{\text{\emph{def}}}(\tau),u^*_s(\tau)\big)-f\big(\tau,x_i^{\text{\emph{def}}}(\tau),u_i^{\text{\emph{def}}}(\tau)\big)\right] \notag \\
\text{subject to } \tau \in [t_i,t_i+T] \text{.} \notag
\end{gather}
Notice that the cost \eqref{Jt} is actually the ergodic mode insertion gradient evaluated at the optimal schedule $u_s^*(\cdot)$. Thus, minimization of \eqref{Jt} is equivalent to selecting the infinitesimal action from $u_s^*(\cdot)$ that will generate the greatest cost reduction relative to only applying default control.

\subsubsection{Determine control duration $\lambda_A$}
\label{howlong}
The final step in synthesizing an ergodic control action is to choose how long to act, i.e., a finite control duration $\lambda_A$, such that condition~\eqref{improve_condition} is satisfied.  From \cite{egerstedt2006transition,caldwell2012projection}, there is a non-zero neighborhood around $\lambda \to 0^+$ where the mode insertion gradient models the change in cost in \eqref{improve_condition} to first order, and thus, a finite duration $\lambda_A$ exists that guarantees descent. In particular, for \emph{finite} durations $\lambda$ in this neighborhood we can write 
\begin{align}
\label{djdlam_first_order}
J_{\mathcal{E}}\big(x(t;x_{i}, u_i^*(\cdot))\big)-J_{\mathcal{E}}\big(x(t;x_{i}&, u_i^{\text{\emph{nom}}}(\cdot))\big) \notag\\
&=\Delta J_{\mathcal{E}} \approx \frac{\partial J_{\mathcal{E}}}{\partial \lambda} \bigg \vert_{\tau_A} \lambda \text{.} 
\end{align}
Then, a finite action duration $\lambda_A$ can be calculated by employing a \emph{line search} process \cite{caldwell2012projection}.

After computing the duration $\lambda_A$, the control action $A$ is fully specified (it has a value, an application time and a duration) and thus the solution $u_i^*(t)$ of problem $\mathcal{P}_{\mathcal{E}}$ has been determined. By iterating on this process (Section~\ref{prediction} until Section~\ref{howlong}), we rapidly synthesize piecewise continuous, input-constrained ergodic control laws for nonlinear systems.

\begin{algorithm}
	\normalsize{
		%\centering
		%\framebox{\parbox{3.3in}{
		\caption{ Reactive RHEE for varying $\Phi(x)$}
		\label{reactive_rhee} 
		\hrule		
		\vspace{2 mm}			
		\textbf{Define} ergodic memory $M_{erg}$, distribution sampling time $t_{\phi}$.\\		
		\textbf{Initialize} current time $t_{curr}$, current state $x_{curr}$. 
		
		\hrule			
		\vspace{2 mm}  
		While $t_{curr}<\infty$
		
		\begin{enumerate}
			\item Receive/compute current $\Phi_{curr}(x)$.  
			\item $t_0^{erg} \leftarrow t_{curr} - M_{erg}$
			\item $t_{final} \leftarrow t_{curr}+t_{\phi}$
			\item $\bar{x}_{cl} = RHEE(t_{curr},x_{curr},t_0^{erg},t_{final},\Phi_{curr}(x))$   
			\item $t_{curr} \leftarrow t_{curr} + t_{\phi}$
			\item $x_{curr} = \bar{x}_{cl}(t_{final})$  
		\end{enumerate}
		
		end while	
		
	}
	% }}
\end{algorithm}

 The Receding-Horizon Ergodic Exploration (RHEE) process for a dynamically varying $\Phi(x)$ is given in Algorithm~\ref{reactive_rhee}. The re-initialization of RHEE when a new $\Phi(x)$ is available serves two purposes: first, it allows for the new coefficients $\phi_k$ to be calculated\footnote{This corresponds to the general case when the distribution time evolution is unknown. If, however, $\Phi(x)$ is a time-varying distribution with known evolution in time, we can pre-calculate the coefficients $\phi_k$ offline to reduce the computational cost further.} and second, it allows the update of the ergodic initial time $t_0^{erg}$.

The ergodic initial time $t_0^{erg}$ is particularly important for the algorithm performance because it regulates how far back in the past the algorithm should look when comparing the spatial statistics of the trajectory (parameterized by $c_k$) to the input spatial distribution (parameterized by $\phi_k$). If the spatial distribution is regularly changing to incorporate new data (for example in the case that the distribution represents expected information density in target localization as we will see in Section~\ref{localization}), it is undesirable for the algorithm to use state trajectory data all the way since the beginning of exploration. At the same time, ``recently'' visited states must be known to the algorithm so that it avoids visiting them multiple times during a short time frame. To specify our notion of ``recently'' depending on the application, we use the parameter $M_{erg}$ in Algorithm~\ref{reactive_rhee} which we call ``ergodic memory'' and simply indicates how many time units in the past the algorithm has access to, so that it is $t_0^{erg} = t_0 - M_{erg}$ every time RHEE is re-initialized at time $t_0$. 

The C++ code for both Algorithms \ref{rhee} and \ref{reactive_rhee} is available online at  \url{github.com/MurpheyLab}.

\paragraph{Remarks on computing $c_k^i$}
\label{compute_ck}
The cost $J_{\mathcal{E}}$ in \eqref{ergodic_cost} depends on the full state trajectory from a defined initial time $t=t_0^{erg}\leq t_i$ in the past (instead of $t_i$ as in common tracking objectives) to $t=t_i+T$, which could arise concerns with regard to execution time and computational cost. Here, we show how to compute $c_k^i$ in a way that avoids integration over an infinitely increasing time duration as $t_i \rightarrow \infty$.  
 To calculate trajectory coefficients $c_k^i$ at time step $t_i$ with $k\in  \mathcal{K}$, and thus cost $J_{\mathcal{E}}$, notice that:
 
 \small 
\begin{align}
\label{calcck}
c_k^i& =  \frac{1}{t_i+T-t_0^{erg}}\int\limits_{t_0^{erg}}^{t_i+T}{F_k(x(t)) dt} =\\
&= \underbrace{ \frac{1}{t_i+T-t_0^{erg}} \int\limits_{t_0^{erg}}^{t_i}{F_k(x(t)) dt} }_{\bar{c}_k^{(i)}}+ \frac{1}{t_i+T-t_0^{erg}} \int\limits_{t_i}^{t_i+T}{F_k(x(t)) dt} \notag\\
&\text{where recursively  } \notag \\
& \bar{c}_k^{(i)} =  \frac{t_{i-1}+T-t_0^{erg}}{t_i+T-t_0^{erg}} \bar{c}_k^{(i-1) } + \frac{1}{t_i+T-t_0^{erg}} \int\limits_{t_{i-1}}^{t_i}{F_k(x(t)) dt}\;\; \notag\\
& \forall \;\; i\geq 1, \;\; k \in  \mathcal{K} 
 \text{  with  } \bar{c}_k^{(0)} = 0. \notag 
\end{align}

\normalsize
Therefore, only the current open-loop trajectory $x(t)$ for all $t \in [t_i, t_i+T]$ and a set of $(K+1)^\nu$ coefficients $\bar{c}_k^{(i)} \in \mathbb{R}$ are needed for calculation of $c_k^i$ and thus  $J_{\mathcal{E}}$ at the $i^{th}$ time step. Coefficients $\bar{c}_k^{(i)} \in \mathbb{R}$ can be updated at the end of the $i^{th}$ time step and stored for use in the next time step at $t_{i+1}$. This provides the advantage that although the cost depends on an integral over time with an increasing upper limit, the amount of stored data needed for cost calculation does not increase but remains constant as time progresses.

\subsection{Stability analysis}
\label{stab}

In this section, we establish the requirements for ergodic stability of the closed-loop system resulting from the receding-horizon strategy in Algorithm~\ref{rhee}. To achieve closed-loop stability for Algorithm~\ref{rhee}, we apply a contractive
constraint \cite{de2000contractive,camponogara2002distributed,xie2008first,ferrari2009model}  on the cost. For this reason, we define $\mathcal{C}_{\mathcal{E}}$ from the open-loop problem \eqref{open_loop_problem} as follows. 

\begin{Definition}
\label{contractive2}
	Let $\mathcal{Q}$ be the set of trajectories $x(\cdot):  \mathbb{R} \rightarrow \mathcal{X}$ in \eqref{f} and $\mathcal{Q}_d \subset \mathcal{Q}$ the subset that satisfies $c_k(x(\cdot)) - \phi_k = 0$ for all $k$. 
Suppose $L(x(\cdot),u,t): \mathcal{Q} \times \mathbb{R}^m \times \mathbb{R} \rightarrow \mathbb{R} $ is defined as follows:
\small
\begin{align}
\label{el}
L(x(\cdot),u,t) : = & \frac{2 Q}{t-t_0^{erg}} \mathlarger{\sum}\limits_{k} \Bigg \{\Lambda_k \bigg[c_k(x(\cdot),t)-\phi_k \bigg] \cdot \notag \\ 
& \cdot\bigg[ F_k(x(t)) -c_k(x(\cdot),t) + f(x(t),u,t)^T \mathlarger{\int}\limits_{t_0^{erg}}^{t} \frac{\partial F_k(x(s))} {\partial x(s) } ds   \bigg  ] \Bigg\} 
\end{align}
\normalsize	 
where $c_k(x(\cdot),t) = \frac{1}{t-t_0^{erg}}\int  _{t_0^{erg}}^{t} F_k(x(s))ds$ denote the Fourier-parameterized spatial statistics of the state trajectory up to time $t$. Through simple computation, we can verify that $L(x_{d}(\cdot),0,\cdot) = 0$ when $x_{d}(\cdot) \in \mathcal{Q}_d$.
Then, 	the ergodic open-loop problem improves the ergodic cost at each time step by an amount specified by the condition \eqref{improve_condition} with $\mathcal{C}_{\mathcal{E}}$ defined as 
\begin{equation}
\label{ergodic_condition}
\mathcal{C}_{\mathcal{E}} = \int_{t_{i-1}+T}^{t_{i}+T} L\big(x_{i}^{\text{\emph{def}}}(\cdot), u_{i}^{\text{\emph{def}}}(t) , t \big) \,dt. 
\end{equation} 
\end{Definition}

This choice of $\mathcal{C}_{\mathcal{E}}$ allows us to rewrite expression \eqref{improve_condition} of the open-loop problem, as a contractive constraint used later in the Proof of Theorem~\ref{stab_theorem}. Contractive constraints have been widely used in the MPC literature to show closed-loop stability as an alternative to methods relying on a terminal
(region) constraint  \cite{grune2011nonlinear,lee2011model,mayne2014model}. Conditions
similar to the contractive constraint used here (see \eqref{equiv_condition})  also appear in terminal region methods \cite{grune2011nonlinear,lee2011model,mayne2014model}, either in continuous or in discrete time, as an intermediate step used to prove closed-loop stability.

Next, we define stability in the ergodic sense\footnote{Note how this definition differs from the definition of asymptotic stability about an equilibrium point as we now refer to stability of a motion instead of stability of a single point.}. 
\begin{Definition}
\label{stability}
Let $\mathcal{X}^{\nu} \subset \mathbb{R}^{\nu}$ be the set of states to be ergodically explored. The closed-loop solution $x_{\nu}(t) : \mathbb{R} \rightarrow \mathcal{X}^{\nu}$  resulting from an ergodic control strategy   applied on \eqref{f} is ergodically stable  if the  difference $C(x)-\Phi(x)$ for all $x$ with $C(x)$ defined in \eqref{xstat} (see Section~\ref{erg_section}) converges to a zero density function $0(x)$. Using Fourier parameterization as shown in equations \eqref{phik} and \eqref{info_states}, this requirement is equivalent to \mbox{$c_k(x_{\nu})-\phi_k(x_{\nu}) \rightarrow 0 $} for all $k$ as $t \rightarrow \infty$. 	
\end{Definition}

The following assumptions will be necessary in proving stability. 

\begin{Assumption}
The dynamics $f$ in \eqref{f} are continuous in $u$, piecewise-continuous in $t$, and continuously differentiable in $x$. Also, $f$ is compact, and thus bounded, on any given compact sets $\mathcal{X}$ and $\mathcal{U}$. Finally, $f(\cdot,0,0)=0$.
\end{Assumption}

\begin{Assumption}
	\label{mayer}
 There is a continuous positive definite---with respect to the set $\mathcal{Q}_d$---and radially unbounded function $\mathcal{M}: \mathcal{Q} \times \mathbb{R} \rightarrow \mathbb{R}_+$ such that $L(x(\cdot),u,t)\geq \mathcal{M}(x(\cdot),t)$ for all $u \in \mathbb{R}^m$. 
\end{Assumption}

Assumption~\ref{mayer} is necessary to show that the integral $\int_{t_0^{erg}}^{t} \mathcal{M} (x(\cdot),s) ds$ is bounded for $t \rightarrow \infty$. This result can be then used in conjunction with a well known lemma in \cite{barbalat1959systemes, michalska1994nonlinear, tzorakoleftherakis2017iterative} to prove convergence in the proof  of the stability theorem that follows.

\begin{Theorem}
	\label{stab_theorem}
	Let assumptions 1-2 hold for all time steps $i$. Then, the closed-loop system resulting from the receding-horizon ergodic control strategy is ergodically stable in the sense that \mbox{$c_k(x_{\nu})-\phi_k(x_{\nu}) \rightarrow 0 $} for all $k$ as $t \rightarrow \infty$. 
\end{Theorem}

\begin{proof}
Note that the ergodic metric \eqref{ergodic_cost} can be written as   $J_{\mathcal{E}} = \mathcal{B}(t_i+T,x(\cdot))$ with \mbox{$\mathcal{B}(t,x(\cdot)):=  Q \sum_{k\in  \mathcal{K}} \Lambda_k \bigg[c_k(x(\cdot),t)-\phi_k \bigg]^2$} with $c_k(x(\cdot),t)$ defined in Definition~\ref{contractive2}. Using this definition and converting  $J_{\mathcal{E}}$  from Mayer to Lagrange form yields $J_{\mathcal{E}} = \int_{t_0^{erg}}^{t_i+T} L(x(\cdot),u,t) dt$ with $L(x(\cdot),u,t) = \frac{d}{dt} \mathcal{B}(t,x(\cdot))$ resulting in the expression in \eqref{el}. Going back to Definition~\ref{contractive2} and the ergodic open-loop problem \eqref{open_loop_problem} in Section~\ref{main_ergodic}, we note that condition \eqref{improve_condition} with $\mathcal{C}_{\mathcal{E}}$ in \eqref{ergodic_condition} is a contractive constraint applied in order to generate actions that sufficiently improve the cost between time steps. To see that this is true, one can rewrite \eqref{improve_condition} as
\begin{align}
\label{equiv_condition}
J_{\mathcal{E}}\big(x_i^*(\cdot)\big) - J_{\mathcal{E}}\big(x_{i-1}^*(\cdot)\big) \leq & -\int \limits_{t_{i-1}}^{t_{i}} L\big(x_{i-1}^*(\cdot), u_{i-1}^*(t), t\big) \,dt \notag \\ & \leq - \int_{t_{i-1}}^{t_{i}} \mathcal{M} (x_{i-1}^*(\cdot),s) ds
\end{align}
since from \eqref{default_control}, \mbox{$u_i^{def}(t) \equiv u_{i-1}^*(t)$} in \mbox{$[t_i,t_{i-1}+T]$}.
This contractive constraint directly proves that the integral $\int_{t_0^{erg}}^{t} \mathcal{M} (x(\cdot),s) ds$ is bounded for $t \rightarrow \infty$, which, according to Barbalat's lemma found, e.g., in \cite{barbalat1959systemes, michalska1994nonlinear, tzorakoleftherakis2017iterative},  guarantees asymptotic convergence, i.e., that $x(\cdot) \rightarrow \mathcal{Q}_d$ or equivalently that $c_k(x(\cdot)) - \phi_k \rightarrow 0 $ as $t \rightarrow \infty$.
\end{proof}

\subsection{Multi-agent ergodic exploration}

Assume we have $N$ number of agents $\zeta = 1,...,N$, each with its own computation, collectively exploring a terrain to track an assigned spatial distribution $\Phi(x)$. Each agent $\zeta$ performs RHEE as described in Algorithm~\ref{rhee}. 
At the end of each algorithm iteration $i$ (and thus every $t_s$ seconds),  each agent $\zeta$ communicates the Fourier-parameterized statistics of their exploration trajectories $c^{i}_{k,\zeta}$ up to time $t_{i}$ to all the other agents. By communicating this information, the agents have knowledge of the collective coverage up to time $t_i$ and can use this to avoid exploring areas that have already been explored by other agents. This ensures that the exploration process is coordinated so that the spatial statistics of the combined agent trajectories collectively match the distribution.

To use this information, each agent $\zeta$ updates its trajectory coefficients $c^i_{k,\zeta}$ at time $t_i$ to include the received coefficients $c^{i-1}_{k,j}$  from the previous algorithm iteration $i-1$ so that now the collective agent coefficients $\mathbf{c}_{k,\zeta}^i$ are defined to be:
\begin{align}
\label{multick}
\mathbf{c}_{k,\zeta}^i = c^i_{k,\zeta} +  \frac{1} {N-1}  \cdot \sum\limits_{j=1, j\neq\zeta}^{N} c^{i-1}_{k,j} 
\end{align}
where $c_{k,\zeta}^i$ are the coefficients of the agent $\zeta$ state trajectory at time step $t_i$ calculated as in \eqref{calcck}, and $c^{i-1}_{k,j} $ are the coefficients of the remaining agents state trajectories at the previous time step $t_{i-1}$ also calculated as in \eqref{calcck}. So now, agent $\zeta$ computes the ergodic cost \eqref{ergodic_cost} at $t_i$ based on all the agents' past trajectories and Algorithm $\ref{rhee}$ is guaranteed to compute a control action that will optimally improve it. Note that expression \eqref{multick} expands to:   

\small
\begin{align}
\label{multi}
\mathbf{c}_{k,\zeta}^i = \frac{1}{t_i+ T - t_0^{erg} } \notag& \int\limits_{t_0^{erg}}^{t_i+T}{F_k(x_{\zeta}(t)) dt} + \\ & \frac{1}{(N-1) (t_{i-1}+ T - t_0^{erg}) }   \sum\limits_{j=1, j\neq\zeta}^{N}  \int\limits_{t_0^{erg}}^{t_{i-1}+T}{F_k(x_{j}(t)) dt} 
\end{align}

\normalsize
with $\mathbf{c}_{k,\zeta}^0 = \frac{1}{t_i+ T - t_0^{erg} }  \int_{t_0^{erg}}^{t_0+T}{F_k(x_{\zeta}(t)) dt}$ where $x_{\zeta}(t)$ with $\zeta =1,...,N$ is the agent $\zeta$ state trajectory. Therefore expression \eqref{multick} calculates the combined statistics of the current agent's trajectory $x_{\zeta}(t)$, $\forall t \in [t_0^{erg},t_i+T]$ and of the state trajectories that all the other agents have executed up to current time $t_i$ and temporarily intend to execute from $t_i$ (now) to $t_{i-1}+T$ based on their open-loop trajectories  at the previous time step $t_{i-1}$. 

Note that  if the states of two or more agents are identical, the matching agents motion degenerates to a single agent motion and the multi-ergodic control approach fails to take advantage of all agents' control authority.  

\begin{figure}[t]
	\centering
	\includegraphics[width=2.4in]{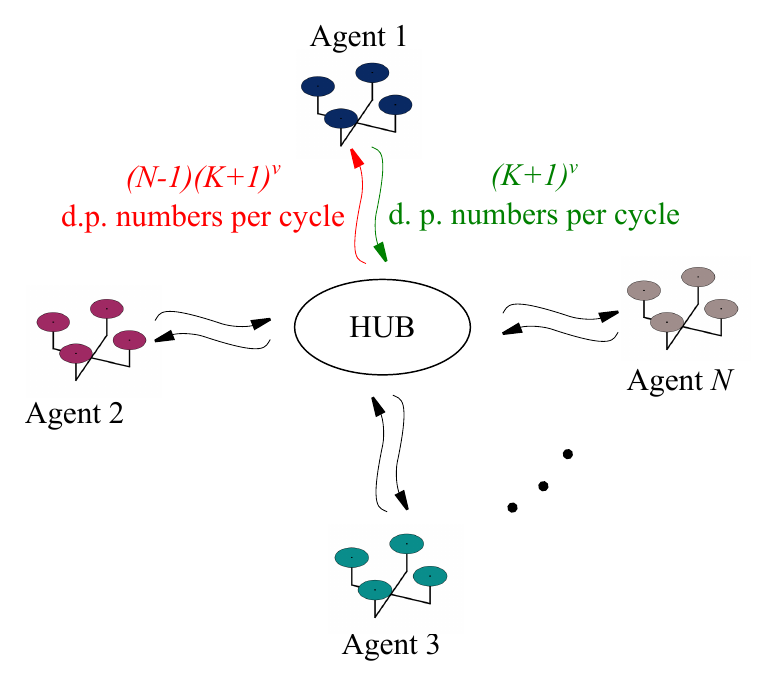} 
	\caption{  Communication network for multi-agent ergodic exploration using a hub configuration. Agents are equipped with independent computational units for local control calculation but exchange information that may influence each other's subsequent actions. The HUB is simply a network component and has no computational capacity.  Assuming that a double precision (d.p.) number has 64 bits and an algorithm cycle completes in $t_s$ seconds, the transmitting bit rate of each individual communication channel should be at least  $\frac{(K+1)^\nu \cdot 64 } {t_s}$ bits/s and receiving bit rate equal or higher than  $(N-1) \frac{(K+1)^\nu \cdot 64 } {t_s}$ bits/s. 
	}
	\label{multi_channel}
\end{figure}

\textit{Computational complexity and communication requirements:} 
This multi-agent ergodic control process exhibits time complexity  $\mathcal{O}(1)$ (i.e., the amount of time required for one algorithm cycle does not scale with $N$) because Algorithm 1 is executed by each agent in \textit{parallel} in a distributed manner.  Computational complexity also remains constant for each agent ($\mathcal{O}(1)$, i.e., the total number of computer operations does not scale with $N$). However, each agent's computational unit needs to communicate with a central transmitter/receiver through a (deterministic) communication channel with receiving capacity that scales linearly in $N$ ($\mathcal{O}(N)$) and with constant transmitting capacity ($\mathcal{O}(1)$).  In particular, at every time step $t_i$, each agent needs to receive $(K+1)^{\nu}$ coefficients corresponding to $c^{i-1}_{k,j} $, by each of the remaining $N-1$ agents. In addition, each agent is responsible to transmit their own  $(K+1)^{\nu}$  coefficients to the rest of the robot network (see Fig.~\ref{multi_channel}). Thus, assuming a constant highest order of coefficients $K$ and number of ergodic variables $\nu$,  transmitting capacity of each agent's communication channel is constant while its receiving capacity scales linearly  with $N$.    
While, in this communication paradigm, we assumed a star network configuration (Fig.~\ref{multi_channel}), note that a fully connected network can also be employed. In any case, the minimum amount of information needed by the team of agents for coordinated exploration is $N$ sets of $(K+1)^{\nu}$ coefficients per algorithm cycle.  Because of this requirement of collective data exchange between agents at each time step, multi-agent ergodic exploration can be characterized as semi-distributed in that each agent executes RHEE (Algorithm~\ref{rhee}) independently but shares information with the other agents after each algorithm cycle.

\section{Receding-horizon ergodic target localization}
\label{localization}
\subsection{Expected Information Density}

Ergodic control for localization of static or moving targets is essentially an application of reactive RHEE in Algorithm~\ref{reactive_rhee} with the following specifics:  1) the agent takes sensor measurements every $t_m$ seconds while exploring the distribution in Step~4, and 2)  belief of targets' state is updated online and used for computation of $\Phi(x)$ in  Step~1.

We focus on the part of calculating $\Phi(x)$ (for now on referred to as Expected Information Density, EID) given the current targets belief and a known measurement model. It is important to point out that the following process for computing the EID depends only on the measurement model; the methodology for belief state representation and update can be arbitrary (e.g., Bayesian methods, Kalman filter, particle filter etc.) and does not alter the ergodic target localization process.   
 The objective is to estimate the unknown parameters $\boldsymbol{\alpha} \in \mathbb{R}^{M} $ describing the $M$ coordinates of a target. We assume that a measurement $\mathbf{z} \in \mathbb{R}^{\mu}$  is made according to a known measurement model 
\begin{equation}
\label{meas_model} 
\mathbf{z} = \Upsilon(\boldsymbol{\alpha}, \mathbf{x})  + \delta,
\end{equation}
 where $\Upsilon(\cdot)$ is a function of
sensor configuration and  parameters, and $\delta$ represents zero mean Gaussian noise with covariance $\Sigma$, i.e., $\delta \sim \mathcal{N}(0,\Sigma)$. 

As in \cite{miller2015ergodic}, we will use the Fisher Information Matrix (FIM) \cite{emery1998optimal, frieden2004science} to calculate the EID. Often used in maximum likelihood estimation, Fisher information $\mathcal{I}(x,\boldsymbol{\alpha})$ is the amount of information a measurement provides at location $x$ for a given estimate of $\boldsymbol{\alpha}$. It   quantifies the ability of a set of random variables, in our case  measurements, to estimate the unknown parameters. For estimation of parameters $\boldsymbol{\alpha} \in \mathbb{R}^{M} $, the Fisher information is represented as a $M\times M$ matrix. Assuming Gaussian noise, the $(i,j)$th FIM element is calculated as  
\begin{equation}
\label{fisher1}
\mathcal{I}_{i,j}(x,\boldsymbol{\alpha})=\frac{\partial{\Upsilon(\boldsymbol{\alpha},x)}}{\partial{\alpha_i}}^T \Sigma^{-1} \frac{\partial{\Upsilon(\boldsymbol{\alpha},x)}}{\partial{\alpha_j}}
\end{equation}
where $\Upsilon(\boldsymbol{\alpha}, \mathbf{x}): \mathbb{R}^M \times \mathbb{R}^n \rightarrow \mathbb{R}^{\mu}$ is the measurement model with Gaussian noise of covariance $\Sigma \in \mathbb{R}^{\mu}$.  Since the estimate of the target position $\boldsymbol{\alpha}$ is represented as a probability distribution function $p(\boldsymbol{\alpha})$, we take the expected value of each element of $\mathcal{I}(x,\boldsymbol{\alpha})$ with respect to the joint distribution $p(\boldsymbol{\alpha})$  to calculate the expected information matrix, $\mathbf{\Phi}_{i,j}(x)$. The $(i,j)$th element of $\mathbf{\Phi}_{i,j}(x)$ is then 
\begin{equation}
\label{exp_info}
\mathbf{\Phi}_{i,j}(x) = \int\limits_{\boldsymbol{\alpha}}^{}  \mathcal{I}_{i,j}(x,\boldsymbol{\alpha})  p(\boldsymbol{\alpha}) \; d\boldsymbol{\alpha}.
\end{equation}
To reduce computational cost, this  integration is performed numerically by discretization of the estimated parameters on a grid and a double nested summation. Note that target belief $p(\boldsymbol{\alpha})$ might incorporate estimates of multiple targets depending on the application. For that reason, this EID derivation process is independent of the number of targets and method of targets belief update. 

In order to build a density map using the information matrix \eqref{exp_info}, we need a metric so that each state $x$ is assigned a single information value.  We will use the following mapping: 
\begin{equation}
\label{fisher2}
\Phi(x) = \det \mathbf{\Phi}(x).
\end{equation}
 The  FIM determinant (D-optimality) is  widely used in the literature, as it is invariant under re-parameterization and linear transformation \cite{ucinski2004optimal}.  A drawback of D-optimality is that it might result in local minima and maxima in the objective function, which makes optimization  difficult when maximizing information. In our case though, local maxima do not pose an issue as our purpose is to approximate the expected information density using ergodic trajectories instead of maximizing it.

\normalsize

% % % % % % % % % % % % % % % % % % % % % % % % % % % % % % % % % % % % % % % %

%\begin{equation}
%\mathbf{\Phi}_{i,j}(x) = \int\limits_{\boldsymbol{\alpha}}^{}  \mathcal{I}_{i,j}(x,\boldsymbol{\alpha}) %\; p(\boldsymbol{\alpha}) \; d\boldsymbol{\alpha}
%\end{equation}

\subsection{Remarks on Localization with Limited Sensor Range}

The efficiency of planning sensor trajectories by maximizing information metrics like the Fisher Information Matrix in  \eqref{fisher1} is highly dependent on the true target location \cite{ucinski2004optimal}: if the true target location is known, the optimized trajectories are guaranteed to acquire the most useful measurements for estimation; if not, the estimation and optimization problems must be solved simultaneously and there is no guarantee that  useful measurements will be acquired especially when the sensor exhibits limited range.  

A limited sensor range serves as an occlusion during localization, in that large regions are naturally occluded while taking measurements. Because of this, how we plan the motion of the agent according to the current target estimate is critical; if, at one point, the current target belief largely deviates from the true target position, the sensor might completely miss the actual target (out of range), never acquiring new measurements in order to update the target's estimate. This would be a possible outcome if we controlled the agent to move towards maximum information (IM). In this section, we explain how receding-horizon ergodic target localization (Algorithm~\ref{reactive_rhee}) with limited sensor range can overcome this drawback under a single assumption.

\begin{Assumption}
	\label{nonzero}
	Let $r \in \mathbb{R}^+$ be the radius defining sensor range so that a sensor positioned at $\boldsymbol{x}_s \in \mathcal{X}_{ \nu}$ can only take measurements of targets  whose true target location $\boldsymbol{\alpha}_{true} \in \mathbb{R}^{\nu} $ satisfies $\| \boldsymbol{x}_s - \boldsymbol{\alpha}_{true} \| _{\nu} < r$. An occlusion $\mathcal{O}$ is defined as the region where no sensing occurs i.e., $\mathcal{O} = \{ x_{s} \in \mathcal{X}_{\nu}: \| \boldsymbol{x}_{s} - \boldsymbol{\alpha}_{true} \| _{\nu} > r\}$. At all times $t_{curr}<\infty$ in Algorithm~\ref{reactive_rhee}, there is $x_q \in \mathcal{X}_{ \nu}$  that simultaneously satisfies $\| \boldsymbol{x}_q - \boldsymbol{\alpha}_{true} \|_{\nu} < r$ and $\Phi_{curr}(x_q)>0$,  where  $\Phi_{curr}(x) \forall x \in \mathcal{X}_{\nu}$ is the expected information density computed as in \eqref{fisher2} at time $t_{curr}$.  
\end{Assumption}

\begin{Proposition}
	\label{conv}
	Let Assumption~\ref{nonzero} hold. Also, let $x_{\nu}(\cdot) : [t_{curr}, \infty) \rightarrow \mathcal{X}_{\nu}$ denote the exploration trajectory  of an agent performing ergodic target localization (Algorithms~\ref{rhee}~and~\ref{reactive_rhee})  with expected information density $\Phi_{curr}(x) \forall x \in \mathcal{X}_{\nu}$, equipped with a sensor of range $r \in \mathbb{R}$. Then, there will be time $t_s \in [t_{curr},\infty)$ where the agent's state satisfies $x_{\nu}(t_s) \in  \mathcal{X}_{\nu} \setminus \mathcal{O}$, so that new measurements are acquired and $\Phi_{curr}(x)$ is updated.	
\end{Proposition}

\begin{proof}	
Due to Assumption~\ref{nonzero}, at time $t_{curr}$  there is $x_q \in  \mathcal{X}_{\nu} \setminus \mathcal{O}$	that satisfies $\Phi_{curr}(x_q)>0$. According to Theorem~\ref{conv} and Definition~\ref{stability}, it is $C(x)-\Phi_{curr}(x) \rightarrow 0$ for all $x\in \mathcal{X}_{\nu}$ as $t \rightarrow \infty$. Therefore, at some time $t \in [t_{curr}, \infty )$, we know that $C(x_q)=\Phi_{curr}(x_q)>0$ that is equivalent to \mbox{$\frac{1}{t-t_{curr}} \int_{t_{curr}}^{t} \delta[x_q-x_\nu(\tau)] d\tau >0$} from Eq.~\eqref{xstat}. This leads to the conclusion that $x_q \in x_{\nu}(\cdot)$ which directly proves the proposition.

\end{proof} 

Assumption 4---stating that information density is always non-zero in an arbitrarily small region around the true target---can be satisfied in various ways. For example, we can adjust  the parameters of the estimation filter to achieve a sufficiently low convergence rate. Alternatively, in cases of high noise and variability, we can artificially introduce nonzero information values across the terrain so as to promote exploration as in the simulation and experimental examples. 
\begin{figure}[b!]
	\centering
	\includegraphics[width=3.5in]{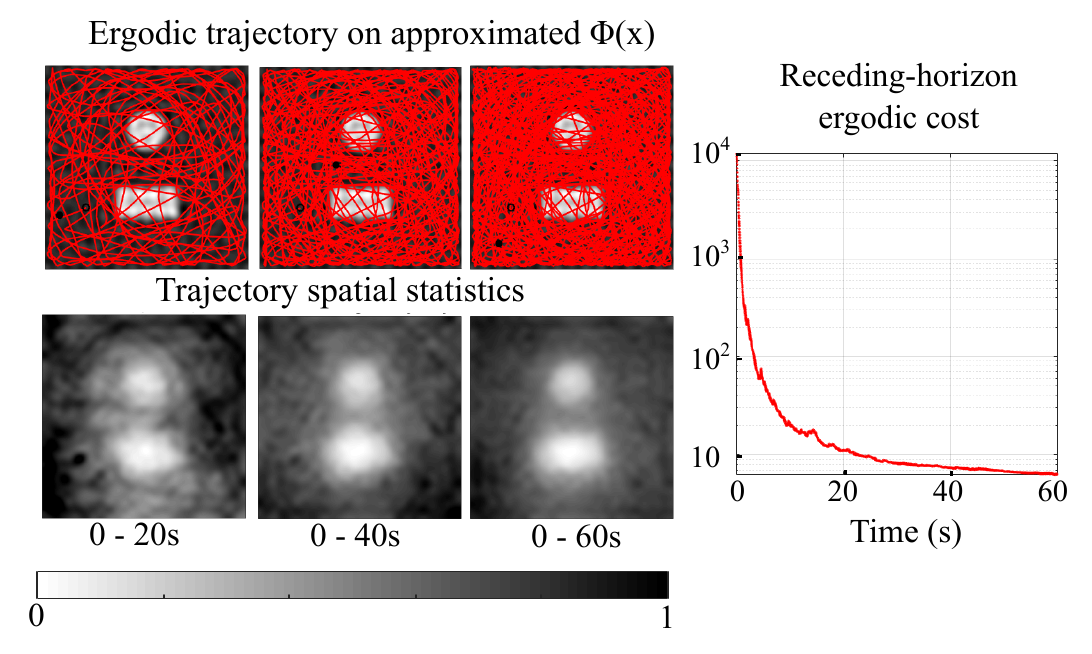} 
	\caption{   Ergodic area coverage in an occluded environment (Algorithm~\ref{rhee} with time horizon $T=0.1s$ and sampling time $t_s=0.02s$). White regions  in $\Phi(x)$ (top row) indicate low to no  probability of detection (occlusions), for example due to sensor failure or physical entities obscuring visibility.  Note that occlusions are not obstacles that should be completely avoided. 
		Bottom row shows the spatial statistics $\Phi_x^i(x)$ of the followed trajectory from $t=0$ to $t=t_i$ calculated as  $\Phi_x^i(x) = \sum_{k\in  \mathcal{K}} \big \{\Lambda_k c_k^i F_k(x) \big \}$ with $\nu=2$ and $K=20$. By the end of the simulation at $t=60$, the trajectory spatial statistics $\Phi_x^{60}(x)$ closely match the initial terrain spatial distribution  $\Phi(x)$, accomplishing the objective of ergodicity as expected. The ergodic cost \eqref{ergodic_cost} is shown to decrease on logarithmic scale over time. Small cost fluctuations result from numerical errors. 
	}
	\label{uniform}
\end{figure}

\section{Simulation Results}

\begin{figure*}[t]
	\centering
	\includegraphics[width=7.0in]{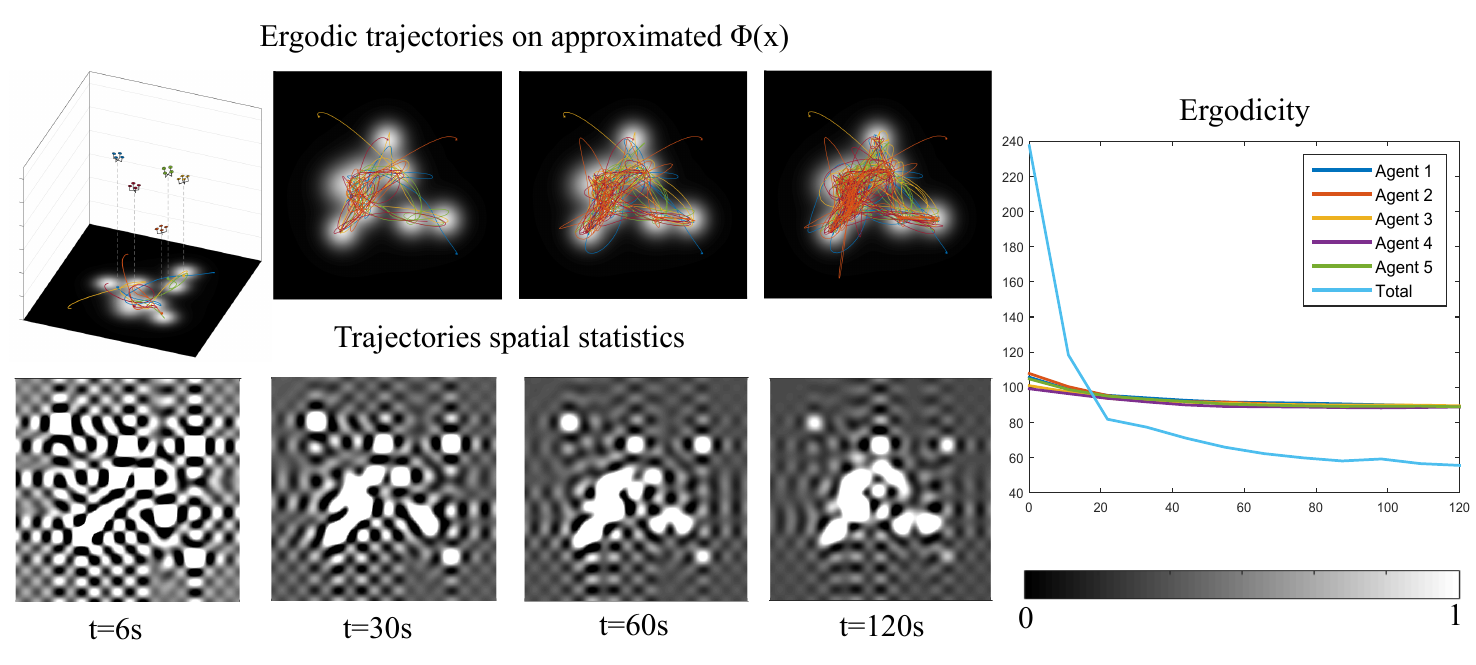} 
	\caption{ Multi-agent UAV exploration (each agent executes Algorithm~\ref{rhee} with trajectory coefficients calculated as in \eqref{multick} with $N=2$, time horizon $T=1.3s$ and sampling time $t_s=0.1s$). Five quadrotor models  collectively explore a terrain to track a spatial distribution $\Phi(x)$ (top row). Highest order of coefficients is $K=12$. Note that agents naturally avoid exploring the same region simultaneously and only return back to already visited areas when sufficient time has passed. Bottom row shows the spatial statistics of the combined agent trajectories calculated as described in Fig.~\ref{uniform} caption. As expected, by the end of simulation, the collective spatial statistics match closely the initial spatial distribution. Plot on the right shows the ergodicity measure of trajectories  as they evolve in time. Ergodicity of each agent's trajectory at time $t$ is calculated as $\sum_{k} \{\Lambda_k [\frac{1}{t-t_0^{erg}}\int_{t_0^{erg}}^{t}{F_k(x_{\zeta}(s)) ds}-\phi_k ]^2 \}$   for the ergodic trajectories $x_{\zeta}(t)$ with $\zeta =1,...,N$. Total ergodicity of the collective trajectories is calculated as  $\sum_{k} \{\Lambda_k [\frac{1}{t-t_0^{erg}}\int_{t_0^{erg}}^{t}{\sum_{j=1}^{N} F_k(x_{j}(s)) ds}-\phi_k ]^2 \}$. A video representation of this exploration process is available in the supporting multimedia files.
	}
	\label{2_quad_tri}
\end{figure*} 

\subsection{Ergodic Exploration and Coverage}
\label{ex1}

\subsubsection{Motivating example - Uniform area coverage with occlusions}

In this first example, we control an agent to explore an occluded environment in order to achieve a uniform probability of detection across a square terrain, using Algorithm~\ref{rhee}. The shaded regions (occlusions) $\mathcal{O}$ comprise a circle and a rectangle in Fig.~\ref{uniform} and  they exhibit zero probability of detection i.e., $\Phi(x)=0 \forall x\in\mathcal{O}$. Such situations can arise in vision-based UAV exploration with occlusions corresponding to shaded areas that limit  visibility, or in surveillance by mobile sensors where  the shaded regions can be thought of as areas where no sensor measurements can be made due to foliage. It is assumed that the agent has second-order dynamics with $n=4$ states $x=[x_1,x_2,x_3,x_4]$, $m=2$ inputs $u=[u_1,u_2]$, and $f(x,u)=[x_2, u_1, x_4, u_2]$ in \eqref{f}. Forcing saturation levels are set as $u_{min} = -50$ and $u_{max} = 50$. 

Snapshots of the agent exploration trajectory is shown in Fig.~\ref{uniform}. As time progresses from $t=0$ to $t=60$ the spacing between the trajectory lines is decreasing, meaning that the agent successfully and completely covers the square terrain  by the end of the simulation. This is also reflected in the spatial statistics of the performed trajectory that eventually closely match the desired probability of detection. A similar example was used by Mathew and Mezi{\'c} in \cite{mathew2011metrics} for evaluation of their ergodic control method that was specific to double integrator systems. Our results serve as proof of concept, showing that Algorithm~\ref{rhee}---although designed to control complex nonlinear systems---can still handle simple systems efficiently and achieve full area coverage, as expected.

\begin{figure*}[t!]
	\centering
	\includegraphics[width=7.1in]{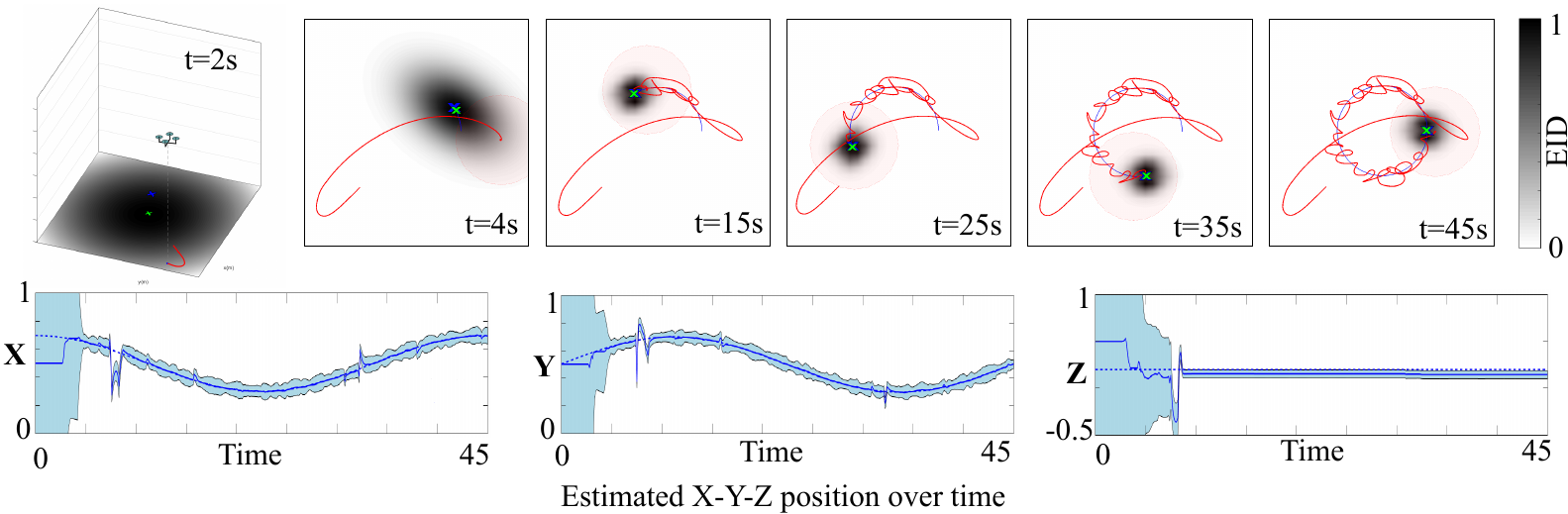} 
	\caption{ Bearing-only localization of a moving target. Top: Top-view snapshots of the UAV trajectory (red curve) where the true target position (blue X-mark) and path (blue curve), and the estimated target position (green X-mark) are also illustrated. The quadrotor can acquire vision-based measurements with a limited range of view that is illustrated as a light red circle around the current UAV position. No prior behavior model of the target motion is available for estimation using EKF. The limited sensor range serves as an occlusion as it naturally occludes large regions while taking measurements. The highest order of coefficients is $K=10$. The quadrotor explores the areas with highest information to acquire useful measurements. Although the geometry of the paths is not predefined, the resulting trajectories follow a cyclic, swirling pattern around the true target position, as one would naturally expect --- like in standoff tracking solutions for example \cite{summers2009coordinated}. Bottom: The target estimate (solid blue curve) is compared to the real target position (dashed blue curve) along with an illustration of the belief covariance (light blue area around estimated position) over time. The target belief converges to a normal spatial distribution with the mean at the true target position and low covariance. A video representation of this exploration process is available in the supporting multimedia files.   
	}
	\label{mov_target}
\end{figure*}

\subsubsection{Multi-agent aerial exploration}
The previous example showed how RHEE can perform area coverage using the simple double-integrator dynamic model. Here and for the rest of this section, we will utilize a 12-dimensional quadrotor model to demonstrate the algorithm's efficiency in planning trajectories for agents governed by higher-dimensional nonlinear  dynamics. The search domain is two-dimensional with $\nu =2$. The quadrotor model \cite{carrillo2013modeling, mellinger2011minimum, luukkonen2011modelling}  has 12 states ($n=12$ in system \eqref{f}), consisting of the position $[x_q,y_q,z_q]$ and velocity $[\dot{x}_q, \dot{y}_q,\dot{z}_q]$ of its center of mass in the inertial frame, and the roll, pitch,  yaw angles $[\phi_q,\theta_q,\psi_q]$  and corresponding angular velocities $[\dot{\phi}_q,\dot{\theta}_q, \dot{\psi}_q]$  in the body frame.  Each of the 4 motors produces the force $u_i$, $i=1,...,4$ ($m=4$ in \eqref{f}), which is proportional to the square of the angular speed, that is, $u_i = k \omega^2$. Saturation levels are set as $u_{min}=0$ and $u_{max}=12$ in \eqref{f}. Nominal control $u_i^{\text{\emph{nom}}}$ from equation~\eqref{default_control} in Algorithm~\ref{rhee} is a PD (proportional-derivative) controller that regulates the agent's height to maintain a constant value.

We use five aerial vehicles to collectively explore a terrain based on a constant  distribution of information. 
The resulting agent trajectories and corresponding spatial statistics are shown in Fig.~\ref{2_quad_tri}. 
It is important to notice here that each agent is not separately assigned to explore a single  distribution peak (as a heuristic approach would entail) but rather all agents are provided with the same spatial distribution as a whole and their motion is planned simultaneously in real time to achieve best exploration on the areas with highest probability of detection. 

This simulation example  was coded in C++ and executed at a Linux-based laptop with an Intel Core i7 chipset.  Assuming that each quadrotor executes Algorithm~\ref{rhee} in parallel, the execution time of the $120s$ simulation is approximately $\sim70s$ per quadrotor, running  about two times faster than real time. 

\subsection{ Ergodic Coverage and Target Localization}

 \begin{figure*}[t!]
	\centering
	\includegraphics[width=7.1in]{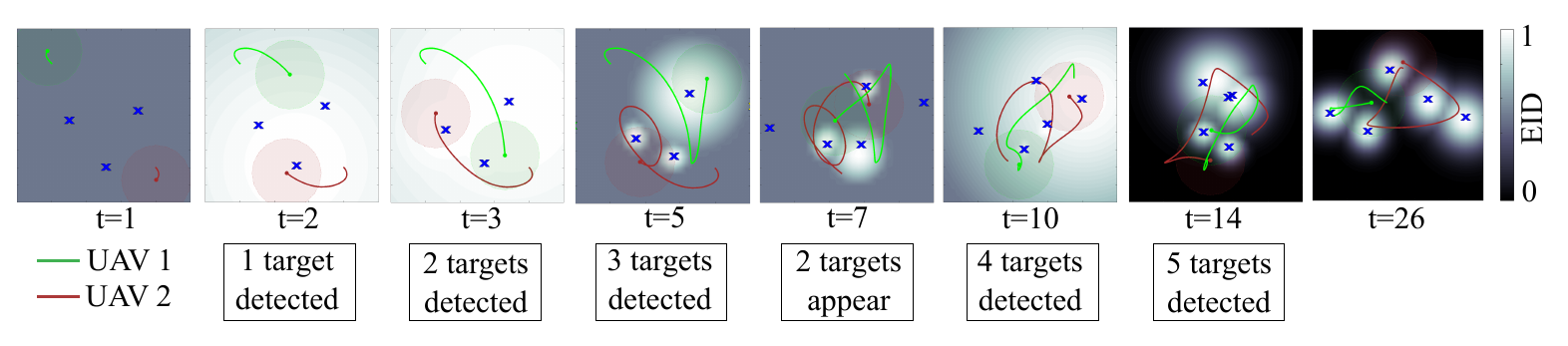} 
	\caption{ Multi-agent simultaneous exploration and targets localization. The problem of exploration vs exploitation is addressed by controlling two agents to localize detected targets while exploring for new undetected targets. The algorithm scales to multiple target localization without any modification, as it tracks a universal non-parametric information distribution instead of each target independently. For cleaner representation, only the UAV trajectories of the past $5s$ are shown in each snapshot. Highest order of coefficients is $K=10$. Mean and standard deviation of targets belief (not shown here) fluctuate in a pattern similar to the experimental results  in Fig.~\ref{exp2b}.  Light green and red circles around the current UAV positions indicate the camera range of view. 	Notice that the EID value is set at a middle level (gray color) in areas where no high information measurements can be taken from the already detected targets. This serves to promote exploration for more targets.  
	}
	\label{simultaneous}
\end{figure*} 

\label{ex2}
In the following examples, we will use the 12-dimensional nonlinear quadrotor model from the previous section to perform motion planning for vision-based static and moving target localization with bearing-only sensing through a gimbaled camera that always faces in the direction of gravity. 
Representing the estimate of target's position as a Gaussian probability distribution function, we use an Extended Kalman Filter (EKF)  \cite{kalman1960new, julier1997new} to update the targets' position belief based on the sensor measurements. We use EKF because it is fast, easy to implement and regularly used in real-world applications but any other estimator (e.g., the Bayesian approaches in \cite{jones2011visual}) can be used instead with no change to the process of Algorithm~\ref{reactive_rhee}. Note for example that reliable bearings-only estimation with EKF cannot be guaranteed as previous results indicate \cite{aidala1979kalman}, so it might be desirable to use a more specialized estimator.

It is assumed that the camera uses  image processing techniques (e.g., \cite{lee2011learning}) to take bearing-only measurements, measuring the azimuth and elevation angles from current UAV position $[x_q,y_q,z_q]$ to the detected target's  position $[x_{\tau}, y_{\tau}, z_{\tau}]$. The corresponding measurement model is $\mathbf{z} $ in \eqref{meas_model} with  
\begin{equation}
\label{bearing}
\Upsilon([x_{\tau}, y_{\tau}, z_{\tau}],[x_q,y_q,z_q]) =\begin{bmatrix}
\tan^{-1}\Big (\frac{x_q-x_{\tau}}{y_q-y_{\tau}} \Big ) & \\[0.3em]
\tan^{-1}\Big (\frac{z_q-z_{\tau}}{\sqrt{ (x_q-x_{\tau})^2 + (y_q-y_{\tau})^2}} \Big ) &  \\[0.3em]
\end{bmatrix}.
\end{equation}
The measurement noise covariance is $\Sigma= diag\{[0.1,0.1]\}$  in radians.  As in the previous examples, a PD controller serves as nominal control, regulating UAV height $z_q$.  The target transition model for EKF is expressed as $\boldsymbol{\alpha}_i = \mathcal{F}(\boldsymbol{\alpha}_{i-1})+ \epsilon$  with $\epsilon$ representing zero  mean Gaussian noise with covariance $C$, i.e., $\epsilon \sim \mathcal{N}(0,C)$. We assume that no prior behavior model of the target motion is available and thus the transition model is $\mathcal{F}(\boldsymbol{\alpha}_{i-1}) \equiv \boldsymbol{\alpha}_{i-1}$. 
Importantly, the camera sensor has limited range of view, completely disregarding targets that are outside of a circle centered at the UAV position with constant radius (as depicted in Fig.~\ref{mov_target}).

\subsubsection{Bearing-only localization of a moving target with limited sensor range}
Here, we demonstrate an example where a quadrotor is ergodically controlled to localize a moving target, with frequency of measurements $f_m=20$Hz and frequency of EID update at $f_{\phi}=10$Hz. The 3D target position  $[x_{\tau}, y_{\tau}, z_{\tau}]$ is localized so that $M=3$. We assume that no prior behavior model of the target motion is available and thus the transition model is $F(\boldsymbol{\alpha}_{k-1}) \equiv \boldsymbol{\alpha}_{k-1}$ with covariance $C =diag\{[0.001,0.001,0.001]\}$ (i.e modeled as a diffusion process). 
Top-view snapshots of the UAV motion are shown in Fig.~\ref{mov_target}. The agent detects the target without prior knowledge of its position, whereafter it closely tracks the target by applying Algorithm~\ref{reactive_rhee} to adaptively explore a varying expected information density $\Phi(x)$. Although the geometry of the paths is not predefined, the resulting trajectories follow a cyclic, swirling pattern around the true target position, as one might expect.

This simulation example  was coded in C++ and executed at a Linux-based laptop with an Intel Core i7 chipset.  The execution time of the $45s$ simulation is approximately $\sim 25s$. This result is representative of the algorithm's computational cost and execution time, because it involves a high-dimensional, nonlinear system.  Localization  is slower than pure exploration, mainly because it requires calculation of the expected information density every $t_{\phi}$ seconds using the expressions \eqref{fisher1}, \eqref{exp_info} and \eqref{fisher2}.

\subsubsection{Multi-agent simultaneous terrain exploration and target localization}
\label{last_sim}

This simulation example is designed to demonstrate search for undetected targets (exploration) and localization of detected moving targets simultaneously, using two agents.  A random number of targets must be detected (exploration) and tracked (target localization) by two UAVs.  Note that here we do not address the issue of cooperative sensing filters \cite{zhang2010cooperative} for multiple sensor platforms: instead, we use a centralized Extended Kalman Filter for simplicity but any filter that provides an estimate of the target's state can be employed instead.
 
  Top-view snapshots of the multi-agent exploration trajectories are given in Fig.~\ref{simultaneous}. At $t=0$ when 3 targets are present in the terrain but none of them have been detected by the agents, the EID is  a uniform distribution across the workspace. Information density is set at  $\Phi(x)=0.5$ for all $x$ (gray coloring). By $t=5$ all present (three) targets have been detected and the EID map is computed based on Fisher Information using expressions \eqref{fisher1}, \eqref{exp_info} and \eqref{fisher2}. Information density is still set at a middle level (instead of zero) in areas where information of target measurements is zero. This serves to promote exploration in addition to localization. In this special case, the terrain spatial distribution $\Phi(x)$ is defined to encode both probability of detection (for the undetected targets) and expected information density (for the detected targets). At $t=7$  two more targets appear and by $t=14$  five targets have been detected. Here, we assume that no more than five targets are to be detected and thus, after the fifth target detection, the spatial distribution only encodes expected information density (note that $\Phi(x)=0$ for all $x$ where information from measurements is zero). 
  
  This simulation example  was coded in C++ and executed at a Linux-based laptop with an Intel Core i7 chipset.  Assuming that each quadrotor executes Algorithm~\ref{rhee} in parallel, the execution time is approximately $\sim 30s$ per quadrotor. 
 
\begin{figure}[t!]
	\centering
	\includegraphics[width=3.3in]{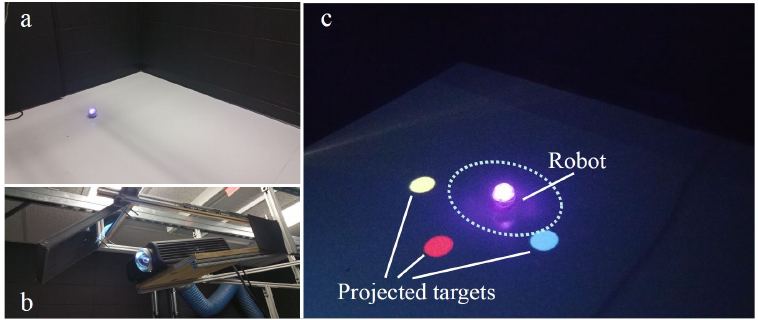} 
	\caption{ 
		(a) The \texttt{sphero SPRK} robot is shown in the experimental setup. The internal mechanism shifts the center of mass by rolling and rotating within the spherical enclosure. RGB LEDs on the top of the \texttt{sphero SPRK} are utilized to track the odometry of the robot through a webcam using OpenCV for motion capture. The Robot Operating System (ROS, available online \cite{ros})  is used to transmit and collect data at 20 Hz. A projection (b) is used to project the targets onto the experimental floor shown in (c).
	}
	\label{exp_setup}
\end{figure}

\begin{figure}[t!]
	\centering
	\includegraphics[width=3.4in]{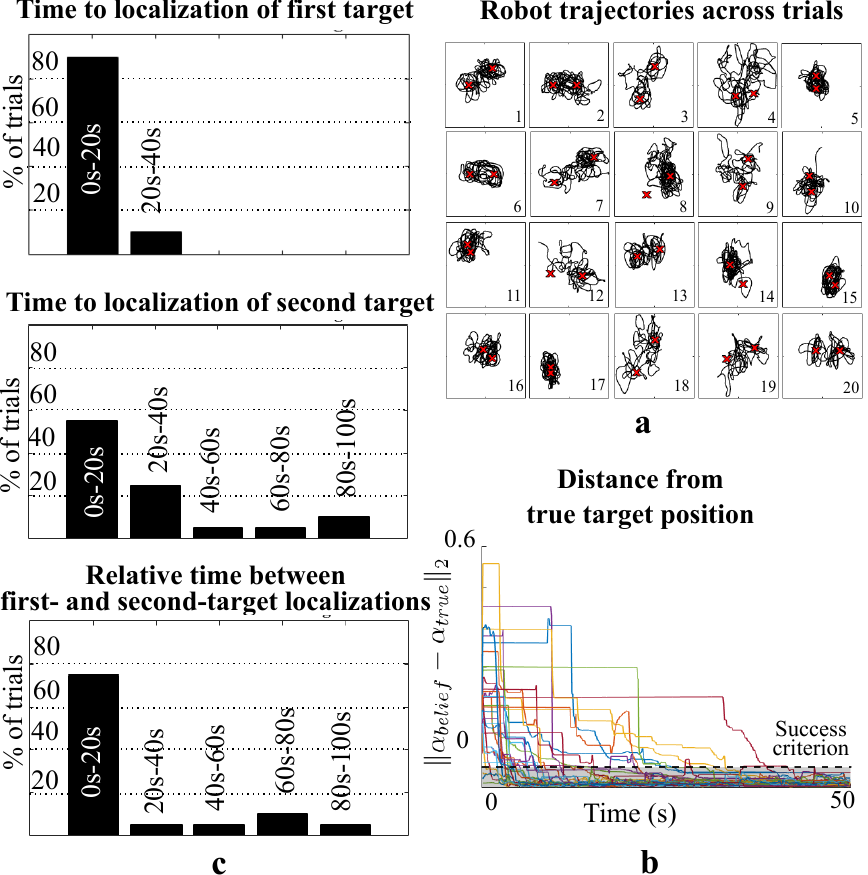} 
	\caption{ Twenty trials of localizing 2 random targets using the \texttt{sphero SPRK} robot at a $1$m$\times$$1$m terrain with simulated limited sensor range of $0.2$m. a) Top-view snapshots of the robot trajectories (black) across trials. The algorithm robustly localizes random pairs of static targets (shown in red) by performing cyclic trajectories around the targets (as required for useful bearing-only measurements), without path specification---the behavior results naturally from the objective of improving ergodicity with respect to the expected information density. b) The distance of the mean target estimate from true target position over time across all trials that were complete by the first 50 seconds. Distance remains constant for as long as the target is outside of the sensor range or it has not be detected yet. c) Bar graphs showing time to localization of first target (top), of second target (center) and relative time that the second target was localized after the first target (bottom), across trials.  The localization of a target is defined to be successful when the $\ell^2$-norm of the difference between the target's position belief and the real target position falls below 0.05, i.e., 	$\| \alpha_{belief} - \alpha_{true} \|_2 < 0.05$. 
	In $100\%$ of the trials, the first target is localized within 40 seconds. In $80\%$ of the trials, both targets are localized by the first $40$ seconds. In $75\%$ of the trials, the waiting time between localizing the first and second target is less than $20$ seconds. 	 Even when target detection is delayed or the EKF fails to converge in a few iterations, the robot is successful in localizing all the targets by $100$ seconds.   
	} 
	\label{exp1}
\end{figure}

\section{Experimental Results} 
\label{experiment}

\begin{figure}[t!]
	\centering
	\includegraphics[width=3.4in]{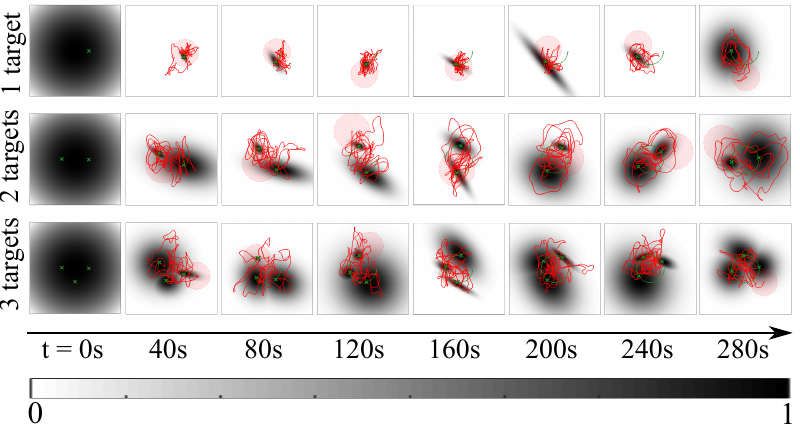} 
	\caption{ 
		 The \texttt{sphero SPRK} robot robustly localizes increasing number of moving targets with bearing-only measurements. The targets' belief---represented as a black-and-white spatial distribution---remains close to the actual targets' location for all 280 seconds. The robot trajectories of a 40-second time window are also shown in red.
	}
	\label{exp2}
\end{figure} 

\begin{figure*}[t!]
	\centering
	\includegraphics[width=7.0in]{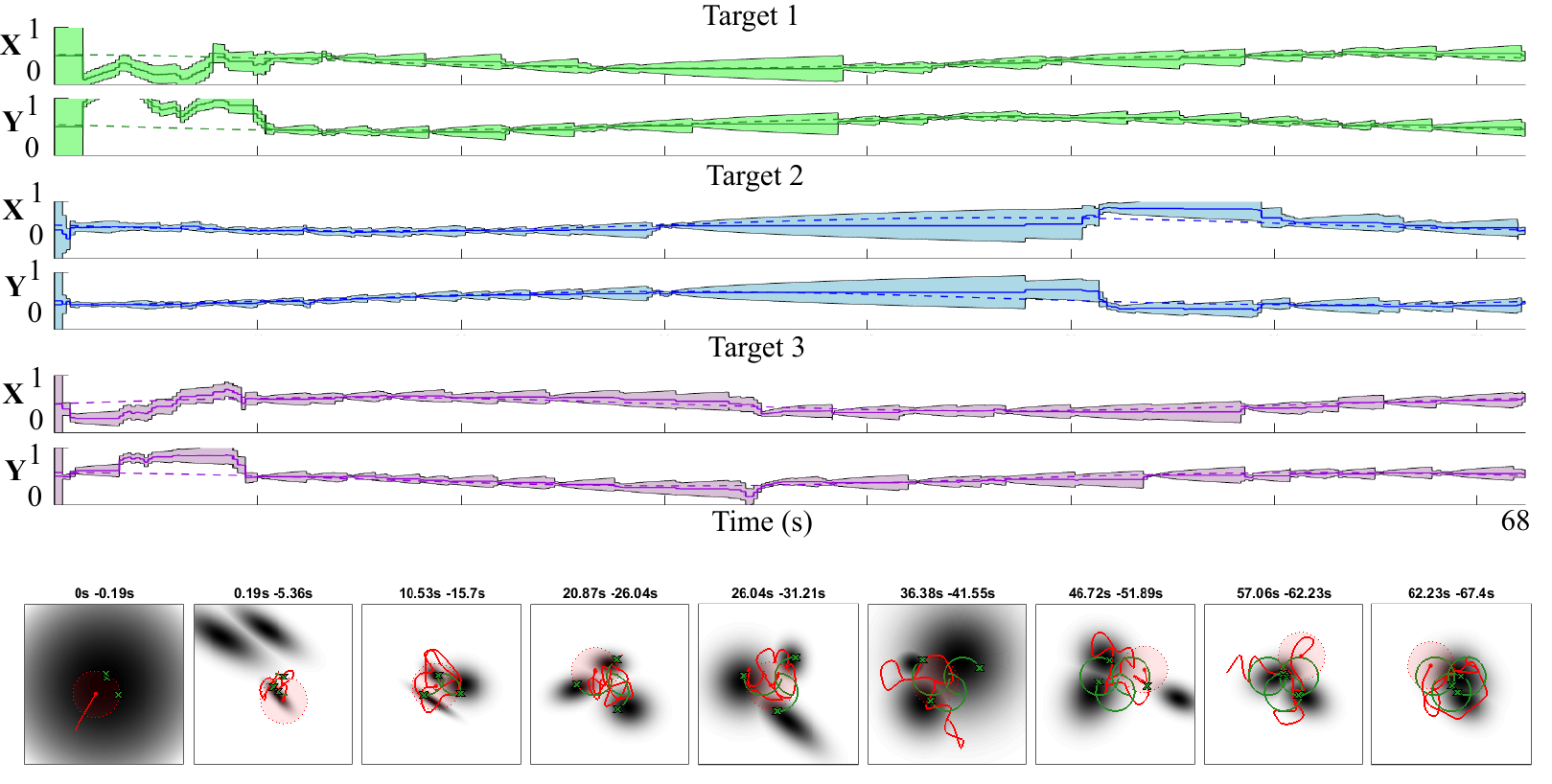} 
	\caption{ Localization of 3 moving targets using the \texttt{sphero SPRK} robot.  The target estimates (solid curves) are compared to the real target locations (dashed curves) along with an illustration of the belief covariance (shaded area around estimated position) over time. Because the targets are constantly moving and the sensor range is limited, the standard deviation of the targets belief  fluctuates as time progresses. The agent localizes each target alternately; once variance on one target estimate is sufficiently low, the agent moves to the next target. Importantly, this behavior  is completely autonomous, resulting naturally from the objective of improving ergodicity.	Note that we can only decompose the targets belief into separate target estimates because of our choice to use EKF where each target's belief is modeled as a normal distribution. This would not be necessarily true, had we used a different estimation filter (e.g., particle filter). Bottom row shows top-view snapshots of the robot and target's motion. A video of this experiment is available in the supporting multimedia files.
	}
	\label{exp2b}
\end{figure*}

We perform two bearing-only target localization experiments using a \texttt{sphero SPRK} robot \cite{sphero} in order to verify real-time execution of Algorithm~\ref{rhee} and showcase the robustness of the algorithm in bearing-only localization.  In addition to the robot, the experimental setup involves an overhead projector and a camera, and is further explained in Fig.~\ref{exp_setup}. The overhead camera is used to acquire sensor measurements of current robot and target positions that are subsequently transformed to bearing-only measurements as in \eqref{bearing}. We additionally simulate a limited sensor range as a circle of $0.2$ m radius  around the current robot position. As in the quadrotor simulation examples, we use an Extended Kalman Filter for bearing-only estimation. In all the following experiments, the ergodic controller runs at approaximately $10$Hz frequency, i.e., $t_s=0.1s$ in Algorithm~\ref{rhee}.

\subsection{Experiment 1}
\label{exp_monte}
In this Monte Carlo experiment, we perform twenty trials of localizing two static targets randomly positioned in the terrain. For each trial, we consider the localization of a target successful when the $\ell^2$-norm of the difference between the target's position belief and the real target position falls below 0.05, i.e.,	$\| \alpha_{belief} - \alpha_{true} \|_2 < 0.05$.  To promote variability, initial mean estimates of target positions are also randomly selected for each trial. Initial distribution $\Phi(x)$ is uniform inside the terrain boundaries. 

The robot simultaneously explores the terrain for undetected targets and localizes detected targets. As in the simulation example of Section~\ref{last_sim}, we achieve simultaneous exploration and localization  by setting the probability of detection (i.e., distribution $\Phi(x)$) across the terrain at a nonzero value. For most trials, targets are successfully localized in less than $60$ seconds.  We see that even in the few cases when target detection is delayed due to limited sensor range or when the EKF fails to converge (as expected for bearing-only models \cite{aidala1979kalman})   (see trials with time to second-target localization higher than $60$s  in Fig.~\ref{exp1}c),  the robot manages to eventually localize both targets by fully exploring the EID instead of moving towards the EID maximum as in information maximization techniques. This result validates Proposition~\ref{conv} that provides convergence guarantees even with poor target estimates and limited sensor range.

\subsection{Experiment 2}
\label{exp_moving}
With this experiment, we aim to demonstrate the robustness of the algorithm in localizing increasing number of moving targets. The resulting robot trajectories for localizing 1, 2 and 3 targets are shown in Fig.~\ref{exp2}, while Fig.~\ref{exp2b} shows the results for localizing 3 targets moving at a different pattern. As in the simulation examples, the motion of each target is modeled as a diffusion process. Note that because the targets are constantly moving and the sensor range is limited, the standard deviation of the targets belief  fluctuates as time progresses (see Fig.~\ref{exp2b}). The agent localizes each target alternately; once variance on one target estimate is sufficiently low, the agent moves to the next target. Importantly, this behavior  is completely autonomous, resulting naturally from the objective of improving ergodicity.

\section{Discussion}

\begin{figure}[b]
	\centering
	\includegraphics[width=3.2in]{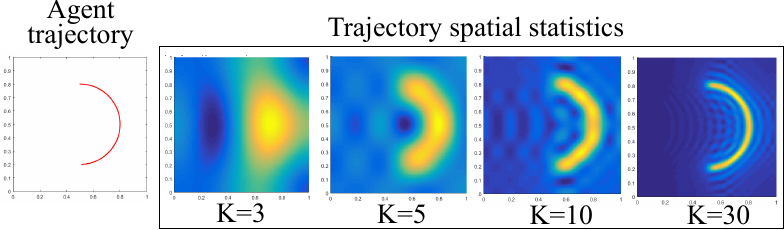} 
	\caption{  The Fourier-parameterized spatial statistics of the semi-circular trajectory shown on the left, calculated for different values of the highest order of coefficients $K$. Yellow indicates high statistical coverage and blue is low to no coverage. 
	}
	\label{coeff_compar}
\end{figure} 

\subsection{Number of Fourier coefficients}
In this subsection, we briefly discuss the effect of $K$, the highest order of coefficients $\phi_k$ in \eqref{phik} and $c_k$ in \eqref{info_states}, on the algorithm's performance. First, note that, because both the trajectory statistics and the desired search density are parameterized by the same number of Fourier coefficients $K$, all choices of $K$ lead to stability as defined in Definition~\ref{stability}. However, since convergence only concerns the trajectory statistics, i.e., \mbox{$c_k(x_{\nu})-\phi_k(x_{\nu}) \rightarrow 0 $}, and not the trajectory itself, different choices of $K$ can affect the computed trajectories. 

Figure~\ref{coeff_compar} shows the spatial statistics of a   semi-circular trajectory with constant velocity, represented with an increasing number of Fourier coefficients. The representation of the agent's trajectory becomes more diffuse, as $K$ decreases. If the search density has fine-scale details, small $K$ might mean that the agent will disregard the details despite meeting the ergodicity requirement. In addition, as expected, there is a diminishing returns property: the rate of change in the algorithm output decreases as $K$ is increased (because the coefficients in the ergodic metric become small as K becomes large). We can see this in Fig.~\ref{coeff_compar} where the trajectory spatial statistics for $K=30$ have nearly converged to the original agent trajectory shape. 

 To sum up, our choice of $K$ depends  on how refined the given search density is.  In the simulation and experimental results presented here, we found that a minimum $K=5$ is sufficient in completing the assigned tasks.

\subsection{Convergence Rate}

 Note from the stability analysis in Section~\ref{stab} that  Algorithm~\ref{rhee}  is formulated so that  an upper bound on the convergence rate is satisfied   in equation~\eqref{equiv_condition}, determined by $\mathcal{M}(x(\cdot),t)$. Although $\mathcal{M}$ is state-dependent, we may assume  an additional constraint $\mathcal{D}(t)$ in Assumption~\ref{mayer} so that $L(x(\cdot),u,t)\geq \mathcal{M}(x(\cdot),t) \geq  \mathcal{D}(t)$.  This imposes a time-dependent convergence rate of the form $J_{\mathcal{E}}\big(x_i^*(\cdot)\big) - J_{\mathcal{E}}\big(x_{i-1}^*(\cdot)\big) \leq-\int_{t_{i-1}}^{t_{i}} \mathcal{D}(t) \,dt$ in \eqref{equiv_condition}. An upper-bound on convergence rate is then the highest rate, determined by $\mathcal{D}(t)$,  for which there always exists a single control action of finite duration that satisfies the constraint requirement, so that open-loop problem \eqref{open_loop_problem} attains a solution at each time step.

One way to satisfy this requirement is by manipulating the sampling time $t_S$. Our theoretic results in \cite{tzorakoleftherakis2017iterative} show that---at least, for Bolza form control objectives---as long as the contractive constraint depends on $t_s$, there exists $t_s$ as $t_s \rightarrow 0$ that guarantees that there is  always a finite-duration single control action to satisfy the constraint requirement.

\section{ Conclusions }
\label{conclusion}

In this paper, we exploited the advantages of hybrid systems theory to formulate a receding-horizon ergodic control algorithm that can perform real-time coverage and target localization, adaptively using sensor feedback to update the expected information density. In target localization, this ergodic  motion planning strategy  controls the robots to track a  non-parameterized information distribution across the terrain instead of individual targets independently, thus being completely decoupled from the estimation process and the number of targets.   We demonstrated---in simulation with a 12-dimensional UAV model and in experiment using the \texttt{sphero SPRK} robot---that ergodically controlled robotic agents can reliably track moving targets in real time based on bearing-only measurements even when the number of targets is not known a priori and the targets' motion is only modeled as a diffusion process. %a model of target behavior is unavailable. 
Finally, the simulation and experiment examples served to highlight the importance of and to verify stability of the ergodic controls with respect to the expected information density, as proved in our theoretical results. 

\appendices

\section{Proof}
\label{proof}
The proof of Proposition~\ref{new_dJ_dl} is as follows. To make explicit the dependence on action $A$, we write inputs \mbox{$u: \mathbb{R} \times \mathbb{R^+} \times \mathbb{R} \times U \rightarrow U$} of the form of $u_{i}^*(t)$ in \eqref{open_loop_problem} as 
\begin{gather}
u(t;\lambda_A,\tau_A, u_A) = \begin{cases} 
      u_A & \tau_A\leq t\leq \tau_A+\lambda_A \\
      u_i^{def} & \text{else.}
   \end{cases} \notag
\end{gather}
When $\lambda_A=0$, it is $u(t;0,\cdot,\cdot)\equiv u_i^{def}$, i.e., no action is applied. Accordingly, we define $\bar{J_{\mathcal{E}}}(\lambda_A, \tau_A, u_A):= J_{\mathcal{E}}(x(t;t_0, x_{0}, u(t;\lambda_A,\tau_A, u_A) ))$ so that the performance cost depends directly on the application parameters of an  action $A$. Assuming $t_0^{erg}=t_0$ and defining $\beta:= \frac{1}{t_i+T-t_0}\int\limits_{t_0}^{t_i+T}{F_k(x(t)) dt}-\phi_k$, it is

\begin{equation}
\label{ergodic_dj_dl}
\frac{\partial J_{\mathcal{E}}}{\partial \lambda} = \frac{\partial J_{\mathcal{E}}}{\partial \beta} \frac{\partial \beta}{\partial \lambda}
\end{equation}
where 
\begin{equation}
 \frac{\partial J_{\mathcal{E}}}{\partial \beta} = 2 Q \sum\limits_{k\in  \mathcal{K}} \Lambda_k \bigg[ \underbrace{ \frac{1}{t_i+T-t_0} \int\limits_{t_0}^{t_i+T}{F_k(x(\sigma)) d\sigma}}_{c_k^i}-\phi_k  \bigg]
\end{equation}
and
\begin{equation}
 \frac{\partial \beta}{\partial \lambda} = \frac{1}{t_i+T-t_0} \int\limits_{\tau}^{t_i+T} \frac{\partial F_k(x(t))}{\partial x(t)} \frac{\partial x(t)}{\partial \lambda} dt
\end{equation}
where the integral boundary changed from $t_0$ to $\tau$  because the derivative of $x(t)$ with respect to $\lambda$ is zero when $t<\tau$. Then, expression \eqref{ergodic_dj_dl} can be rearranged, pulling $\frac{\partial J_{\mathcal{E}}}{\partial \beta}$ into the  integral over $t$, and switching
the order of the integral and summation, to the following:
\begin{equation}
\label{ell}
\frac{\partial J_{\mathcal{E}}}{\partial \lambda} = \int\limits_{\tau}^{t_i+T} \underbrace{\frac{2 Q }{t_i+T-t_0} 
\sum\limits_{k\in  \mathcal{K}} \bigg \{\Lambda_k \big[ c_k^i-\phi_k  \big]
    \frac{\partial F_k(x(t))} {\partial x(t)} \bigg\}}_{\ell(t,x)} \frac{\partial x(t)}{\partial \lambda}             dt
\end{equation}
with\footnote{Expression \eqref{dxdt} is a direct result of applying the fundamental theorem of calculus on the system equations \eqref{f}.}
\begin{equation}
\label{dxdt}
\frac{\partial x(t)}{\partial \lambda} = \Phi(t,\tau) \Big[ f(\tau,x(\tau),u_A)-f(\tau,x(\tau),u_i^{def}(\tau)) \Big]
\end{equation}
where $\Phi(t,\tau)$ is the state transition matrix of the linearized system dynamics \eqref{f} with $A = D_xf$. 
Therefore,
\begin{equation}
\frac{\partial J_{\mathcal{E}}}{\partial \lambda}  = \int\limits_{\tau}^{t_i+T} \ell(t,x) \cdot \Phi(t,\tau) dt \;\cdot \Big[ f(\tau,x(\tau),u_A)-f(\tau,x(\tau),u_i^{def}(\tau)) \Big].
\end{equation}
Finally, notice that $\int\limits_{\tau}^{t_i+T} \ell(t,x) \cdot \Phi(t,\tau) dt$ is the convolution equation for the system
\begin{gather}
%\label{ergodic_rhodot}
\dot \rho_{\mathcal{E}} = -\ell(t,x)^T - D_{x}f\Big(t,x,u_i^{def}\Big)^T \rho_{\mathcal{E}} \\
\text{subject to\;\;} \rho_{\mathcal{E}}(t_i+T) = \mathbf{0} \notag 
\end{gather}
where $\ell$ is defined in \eqref{ell}.
Therefore, we end up with the expression for the mode insertion gradient of the ergodic cost at time $\tau$:
\begin{equation}
\frac{\partial J_{\mathcal{E}}}{\partial \lambda}\bigg|_{\tau} = \rho_{\mathcal{E}}(\tau) \; \Big[ f(\tau,x(\tau),u_A)-f(\tau,x(\tau),u_i^{def}(\tau)) \Big].
\end{equation}


\bibliographystyle{IEEEtran}
\bibliography{ergodic}

\end{document}